\documentclass{article}

     \PassOptionsToPackage{numbers, sort&compress, square}{natbib}

\usepackage[final]{neurips_2023}

\newcommand{\W}{\mathbf{W}}
\newcommand{\gH}{\mathbf{H}}
\newcommand{\M}{\mathbf{M}}
\newcommand{\PHI}{\mathbf{\Phi}}

\usepackage[utf8]{inputenc} 
\usepackage[T1]{fontenc}    
\usepackage{hyperref}       
\usepackage{url}            
\usepackage{booktabs}       
\usepackage{amsfonts}       
\usepackage{nicefrac}       
\usepackage{microtype}      
\usepackage{xcolor}         
\usepackage{bm}
\usepackage{amsmath}

\usepackage{amsthm}
\usepackage{graphics}
\usepackage{graphicx}
\usepackage{subfigure}
\usepackage{multirow}
\usepackage{amssymb}
\usepackage{wrapfig}
\usepackage{float}
\usepackage{enumitem}
\usepackage{makecell}
\usepackage{diagbox}
\usepackage{booktabs}
\usepackage{color, colortbl}
\definecolor{Gray}{gray}{0.9}
\usepackage{capt-of}
\usepackage{varwidth}
\usepackage{enumerate}
\usepackage{slashbox}
\usepackage{threeparttable}
\usepackage[capitalise]{cleveref}       
\usepackage{algorithm}
\usepackage{colortbl}
\usepackage{textcomp}
\usepackage{amsfonts}
\usepackage{algorithmic}
\usepackage{wrapfig}
\usepackage{marvosym}

\definecolor{mediumpersianblue}{rgb}{0.0, 0.4, 0.65}
\definecolor{citecolor}{RGB}{0, 7, 95}
\definecolor{linkcolor}{RGB}{176, 108, 49}
\hypersetup{colorlinks,linkcolor={linkcolor},citecolor={mediumpersianblue},urlcolor={cyan}}

\newtheorem{theorem}{Theorem}

\newtheorem{assumption}{Assumption}
\newtheorem{lemma}{Lemma}

\title{Federated Learning with Bilateral Curation for Partially Class-Disjoint Data}

\author{%
  Ziqing Fan\textsuperscript{1,2}, Ruipeng Zhang\textsuperscript{1,2}, Jiangchao Yao\textsuperscript{1,2}, Bo Han\textsuperscript{3}, Ya Zhang\textsuperscript{1,2}, Yanfeng Wang\textsuperscript{1,2,\Letter} \\
\textsuperscript{1}Cooperative Medianet Innovation Center, Shanghai Jiao Tong University, \\
\textsuperscript{2}Shanghai AI Laboratory, \textsuperscript{3}Hong Kong Baptist University  \\
  \texttt{\{zqfan\_knight, zhangrp, sunarker\}@sjtu.edu.cn}  \\ \texttt{bhanml@comp.hkbu.edu.hk, \{ya\_zhang, wangyanfeng\}@sjtu.edu.cn} 
}

\begin{document}
\maketitle
\vspace{-10pt}
\begin{abstract}
Partially class-disjoint data~(PCDD), a common yet under-explored data formation where each client contributes \textit{a part of classes} (instead of all classes) of samples, severely challenges the performance of federated algorithms. Without full classes, the local objective will contradict the global objective, yielding the angle collapse problem for locally missing classes and the space waste problem for locally existing classes. As far as we know, none of the existing methods can intrinsically mitigate PCDD challenges to achieve holistic improvement in the bilateral views (both global view and local view) of federated learning. To address this dilemma, we are inspired by the strong generalization of simplex Equiangular Tight Frame~(ETF) on the imbalanced data, and propose a novel approach called FedGELA where the classifier is globally fixed as a simplex ETF while locally adapted to the personal distributions. Globally, FedGELA provides fair and equal discrimination for all classes and avoids inaccurate updates of the classifier, while locally it utilizes the space of locally missing classes for locally existing classes. We conduct extensive experiments on a range of datasets to demonstrate that our FedGELA achieves promising performance~(averaged improvement of 3.9\% to FedAvg and 1.5\% to best baselines) and provide both local and global convergence guarantees. Source code
is available at:~\href{https://github.com/MediaBrain-SJTU/FedGELA.git}{https://github.com/MediaBrain-SJTU/FedGELA}.
\end{abstract}

\section{Introduction}
Partially class-disjoint data (PCDD)~\citep{fedrs,fullclass,noniid2} refers to an emerging situation in federated learning~\citep{fedavg, fl_intro1, fl_intro2,ye1,zhang1} where each client only possesses information on a subset of categories, but all clients in the federation provide the information on the whole categories. For instance, in landmark detection~\cite{landmark} for thousands of categories with data locally preserved, most contributors only have a \textit{subset} of categories of landmark photos where they live or traveled before; and in the diagnosis of Thyroid diseases, due to regional diversity different hospitals may have shared and distinct Thyroid diseases ~\citep{rational1}. 
It is usually difficult for each party to acquire the full classes of samples, as the participants may be lack of domain expertise or limited by demographic discrepancy. Therefore, how to efficiently handle the \emph{partially class-disjoint data} is a critical (yet under-explored) problem in real-world federated learning applications for the pursuit of personal and generic interests.

Prevalent studies mainly focus on the general heterogeneity without specially considering the PCDD challenges:
generic federated leaning~(G-FL) algorithms adopt a uniform treatment of all classes and mitigate personal differences by imposing constraints on local training~\citep{moon,fedprox}, modifying logits~\citep{fedlc,fedrs} adjusting the weights of submitted gradients~\citep{fednova} or generating synthetic data~\citep{fedgen}; in contrast, personalized federated learning~(P-FL) algorithms place relatively less emphasis on locally missing classes and selectively share either partial network parameters~\citep{fedrep, fedper} or class prototypes~\citep{fedproto} to minimize the impact of personal characteristics, thereby separating the two topics. Those methods might directly or indirectly help mitigate the data shifts caused by PCDD, however, as far as we know, none of the existing works can mitigate the PCDD challenges to achieve holistic improvement in the bilateral
views~(global and local views) of federated learning. Please refer to Table~\ref{tab:related} for a comprehensive comparison among a range of FL methods from different aspects.

Without full classes, the local objective will contradict the global objective, yielding the angle collapse for locally missing classes and the waste of space for locally existing classes. Ideally, as shown in Figure~\ref{fig:intro}(a), global features and their corresponding classifier vectors shall maintain a proper structure to pursue the best separation of all classes. However, the angles of locally missing classes' classifier vectors will collapse, when trained on each client with partially class-disjoint data, as depicted in Figure~\ref{fig:intro}(b),~\ref{fig:intro}(c). FedRS~\citep{fedrs} notices the degenerated updates of the classifier and pursues the same symmetrical structure in the local by restricting logits of missing classes. Other traditional G-FL algorithms indirectly restrict the classifier by constraining logits, features, or model weights, which may also make effects on PCDD. However, they cause another problem: space waste for personal tasks. As shown in Figure~\ref{fig:intro}(d), restricting local structure will waste feature space and limit the training of the local model on existing classes. P-FL algorithms utilize the wasted space by selectively sharing part of models but exacerbate the angle collapse of classifier vectors. Recent FedRod~\citep{fedrod} attempts to bridge the gap between P-FL and G-FL by introducing a two-head framework with logit adjustment in the G-head, but still cannot address the angle collapse caused by PCDD.
\begin{figure}[t]

\centering
\includegraphics[width=0.8\textwidth]{ 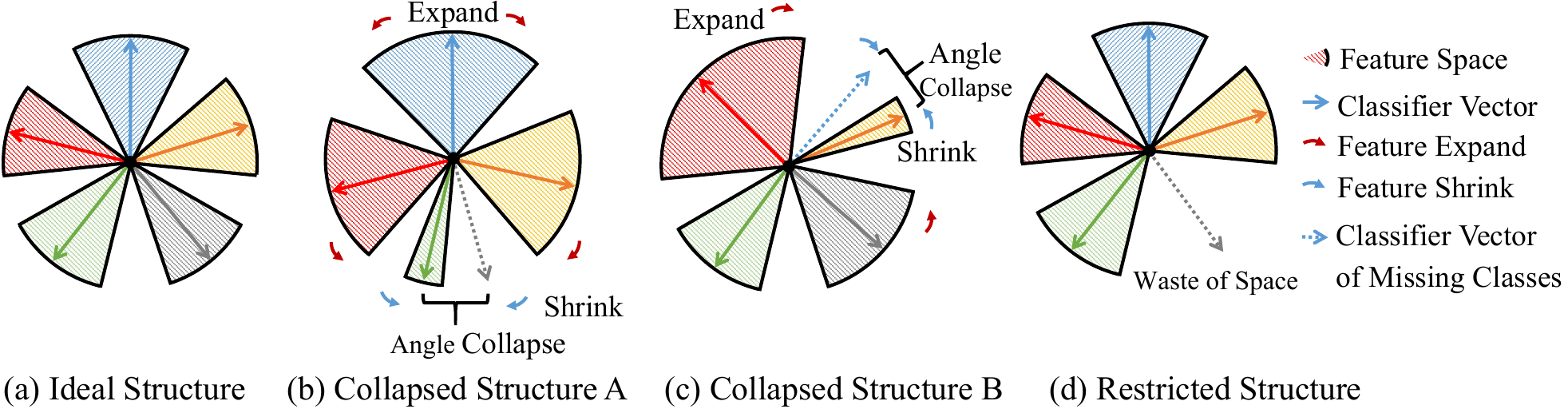}
\vspace{-5pt}

\caption{Illustration of feature spaces and classifier vectors trained on the global dataset, two partially class-disjoint datasets (A and B), and restricted by federated algorithms.  (a) is trained on the globally balanced dataset with full classes. (b) and (c) are trained on datasets A and B, respectively, which suffer from different patterns of classifier angle collapse problems. (d) is averaged in the server or constrained by some federated algorithms.}
\vspace{-5pt}
\label{fig:intro}
\end{figure}

To tackle the PCDD dilemma from both P-FL and G-FL perspectives, we are inspired by a promising classifier structure, namely \emph{simplex equiangular tight frame} (ETF)~\citep{etf1,etf2,minority}, which provides each class the same classification angle and generalizes well on imbalanced data. Motivated by its merits, we propose a novel approach, called \textbf{FedGELA}, in which the classifier is \textbf{Globally} fixed as a simplex \textbf{ETF} while \textbf{Locally} \textbf{Adapted} to personal tasks. In the global view, FedGELA merges class features and their corresponding classifier vectors, which converge to ETF. In the local view, it provides existing major classes with larger feature spaces and encourages to utilize the spaces wasted by locally missing classes. With such a bilateral curation, we can explicitly alleviate the impact caused by PCDD. In a nutshell, our contributions can be summarized as the following three points:
\begin{itemize}[leftmargin=15pt]
\vspace{-7pt}
	\item We study a practical yet under-explored data formation in real-world applications of federated learning, termed as partially class-disjoint data~(PCDD), and identify the angle collapse and space waste challenges that cannot be efficiently solved by existing prevalent methods~(Sec.~\ref{sec:motivation}).
 \vspace{-3pt}
	\item We propose a novel method called FedGELA that classifier is globally fixed as a symmetrical structure ETF while locally adapted by personal distribution~(Sec.~\ref{sec:fedgela}), and theoretically show the local and global convergence analysis for PCDD with the experimental verification~(Sec.~\ref{sec:convergence}). 
     \vspace{-3pt}
	\item We conduct a range of experiments on multiple benchmark datasets under the PCDD case and a real-world dataset to demonstrate the bilateral advantages of FedGELA over the state-of-the-art methods from multiple views like the larger scale of clients and straggler situations~(Sec.~\ref{sec:expmain}). We also provide further analysis like classification angles during training and ablation study.~(Sec.~\ref{sec:further}).
 \vspace{-5pt}
\end{itemize}

\vspace{-5pt}
\section{Related Work}
\vspace{-3pt}
\label{relatedwork}
\subsection{Partially Class-Disjoint Data and Federated Learning algorithms}
\vspace{-3pt}
Partially class-disjoint data is one common formation among clients that can significantly impede the convergence, performance, and efficiency of algorithms in FL~\citep{noniid2}. It belongs to the data heterogeneity case, but does have a very unique characteristic different from the ordinary heterogeneity problem. That is, if only each client only has a subset of classes, it does not share the optimal Bayes classifier with the global model that considers all classes on the server side. Recently, FedRS~\citep{fedrs} has recognized the PCDD dilemma and directly mitigate the angle collapse issue by constraining the logits of missing classes. FedProx~\citep{fedprox} also can lessen the collapse by constraining local model weights to stay close to the global model. Other G-FL algorithms try to address data heterogeneity from a distinct perspective. MOON~\citep{moon} and FedGen~\citep{fedgen} utilizes contrastive learning and generative learning to restrict local representations. And FedLC~\citep{fedlc} introduces logit calibration to adjust the logits of the local model to match those of the global model, which might indirectly alleviate the angle collapse in the local. However, they all try to restrict local structure as global, resulting in the waste space for personal tasks shown in Figure~\ref{fig:intro}(d). P-FL algorithms try to utilize the wasted space by encouraging the angle collapse of the local classifier. FedRep~\citep{fedrep} only shares feature extractors among clients and FedProto~\citep{fedproto} only submits class prototypes to save communication costs and align the feature spaces. In FedBABU~\citep{fedbabu}, the classifier is randomly initialized and fixed during federated training while fine-tuned for personalization during the evaluation. However, they all sacrifice the generic performance on all classes. FedRod~\citep{fedrod} attempts to bridge this gap by introducing a framework with two heads and employing logit adjustment in the global head to estimate generic distribution but cannot address angle collapse. In Table~\ref{tab:related}, we categorize these methods by targets (P-FL or G-FL), skews (feature, logit, or model weight), and whether they directly mitigate the angle collapse of local classifier or saving personal spaces for personal spaces. It is evident that none of these methods, except ours, can handle the PCDD problem in both P-FL and G-FL. Furthermore, FedGELA is the only method that can directly achieve improvements from all views.

\begin{table}[t]
\setlength{\tabcolsep}{7pt}
\caption{Key differences between SOTA methods and our FedGELA categorized by targets~(P-FL or G-FL), techniques~(improve from the views of features, logits or model), and whether directly mitigate angle collapse of classifier vectors or save locally wasted feature spaces caused by PCDD.}\label{tab:related}
\vspace{-4pt}
\small
\centering
\renewcommand\arraystretch{0.05}
\scalebox{0.9}{
\begin{tabular}{c | c | c c c c c}
\toprule[2pt]  \textbf {Target} & \textbf {Research work} & \textbf {Feature View} & \textbf {Logit View} & \textbf {Model View}& \textbf {Mitigate Collapse}& \textbf {Save Space} \\ \midrule
\multirow{4}*[-3.0ex]{\centering G-FL} &  \text { FedProx } & -&- & \checkmark& \checkmark& - \\
\cmidrule{2-7} & \text { MOON } & \checkmark&  - &- & - & -\\ 
\cmidrule{2-7} & \text { FedRS }&  - &  \checkmark&  - & \checkmark & -\\
\cmidrule{2-7} & \text { FedGen }&  \checkmark &  -&  - & \checkmark & -\\
\cmidrule{2-7} & \text { FedLC } & - &\checkmark& - & - & -\\
\midrule 
\multirow{3}{*}[-2.0ex]{P-FL} & \text { FedRep } &\checkmark & -& \checkmark & -& \checkmark\\
\cmidrule{2-7} & \text { FedProto } & \checkmark& -&  \checkmark & -& \checkmark \\
\cmidrule{2-7} & \text { FedBABU } &  \checkmark& -& \checkmark& -& \checkmark\\
\midrule 
\multirow{2}{*}{G\&P-FL} & \text { FedRod } & -& \checkmark&\checkmark & -& \checkmark \\
\cmidrule{2-7} & \text {FedGELA(ours)} & \checkmark& \checkmark&\checkmark & \checkmark& \checkmark \\
\bottomrule[2pt]
\end{tabular}
}
\vspace{-11pt}
\end{table}

\vspace{-10pt}
\subsection{Simplex Equiangular Tight Frame}

The simplex equiangular tight frame~(ETF) is a phenomenon observed in neural collapse~\citep{etf1}, which occurs in the terminal phase of a well-trained model on a balanced dataset. It is shown that the last-layer features of the model converge to within-class means, and all within-class means and their corresponding classifier vectors converge to a symmetrical structure.
To analyze this phenomenon, some studies simplify deep neural networks as last-layer features and classifiers with proper constraints~(layer-peeled model)~\citep{lpm1,lpm2,lpm3,minority} and prove that ETF emerges under the cross-entropy loss. 
However, when the dataset is imbalanced, the symmetrical structure of ETF will collapse~\citep{minority}. 
Some studies try to obtain the symmetrical feature and the classifier structure on the imbalanced datasets by fixing the classifier as ETF~\citep{etf2,lpm2}. Inspired by this, we propose a novel method called FedGELA that bilaterally curates the classifier to leverage ETF or its variants. See Appendix~\ref{ap:A} for more details about ETF. 

\section{Method} \label{sec:method}
\subsection{Preliminaries}
\paragraph{ETF under LPM.}
A typical L-layer DNN parameterized by $\W$ can be divided into the feature backbone parameterized by $\W^{-L}$ and the classifier parameterized by $\W^{L}$. From the view of layer-peeled model~(LPM)~\citep{lpm1,lpm2,lpm3,minority}, training $\W$ with constraints on the weights can be considered as training the C-class classifier $\W^{L}=\{\W_1^L,...,\W_C^L\}$ and features $\gH=\{h^1,...,h^n\}$ of all $n$ samples output by last layer of the backbone with constraints $E_W$ and $E_H$ on them respectively. On the balanced data, any solutions to this model form a simplex equiangular tight frame~(ETF) that
all last layer features $h_c^{i,*}$ and corresponding classifier $\W_c^{L,*}$ of all classes converge as:
{\small
\begin{equation}\label{eq:etf}
\frac{h_c^{i, *}}{\sqrt{E_H}}=\frac{\W_c^{L,*}}{\sqrt{E_W}}=m_c^*,
\end{equation}}
where $m_c^*$ forms the ETF defined as $\M=\sqrt{\frac{C}{C-1}} \mathbf{U}\left(\mathbf{I}_C-\frac{1}{C} \mathbf{1}_C \mathbf{1}_C^T\right).$

Here $\M=\left[m^*_1, \cdots, m^*_C\right] \in \mathbb{R}^{d \times C}, \mathbf{U} \in \mathbb{R}^{d \times C}$ allows a rotation and satisfies $\mathbf{U}^T \mathbf{U}=\mathbf{I}_C$ and $\mathbf{1}_C$ is an all-ones vector. ETF is an optimal classifier and feature structure in the balanced case of LPM. 

\paragraph{FedAvg.} On the view of LPM, given N clients and each with $n_k$ samples, the vanilla federated learning via FedAvg consists of four steps~\citep{fedavg}: 1) In round $t$, the server broadcasts the global model $\W^t=\{\gH^t,\W^{L,t}\}$ to clients that participate in the training~(Note that here $\gH$ is actually the global backbone $\W^{-L,t}$ instead of real features); 2) Each local client receives the model and trains it on the personal dataset. After $E$ epochs, we acquire a new local model $\W^{t}_k$;
3) The updated models are collected to the server as $\{\W^{t}_{1},\W^{t}_{2},\dots,\W^{t}_{N}\}$; 4) The server averages local models to acquire a new global model as $\W^{t+1}=\sum_{k=1}^{N}{p_k \W^t_k}$, where $p_k=n_k\slash\sum_{k'=1}^N n_{k'}$. When the pre-defined maximal round $T$ reaches, we will have the final optimized global model $\W^T$.

\subsection{Contradiction and Motivation}\label{sec:motivation}
\textbf{Contradiction.}~~~In G-FL, the ideal global objective under LPM of federated learning is described as:
{\small$$
\min _{\gH,\W^L} \sum_{k=1}^{N}p_k\frac{1}{n_k} \sum_{c\in C_k} \sum_{i=1}^{n_{k,c}}\mathcal{L}_{CE}\left(h^i_{k,c}, \W^L\right).
$$}
Assuming global distribution is balanced among classes, no matter whether local datasets have full or partial classes, the global objective with constraints on weights can be simplified as:
{\small
\begin{equation}\label{eq:g_objective}
\begin{aligned}
\min _{\gH,\W^L}\frac{1}{n} \sum_{c=1}^C \sum_{i=1}^{n_c}\mathcal{L}_{CE}\left(h^i_{c}, \W^L\right), \text { s.t. } \left\|\W_c^L\right\|^2 \leqslant E_W, \ \left\|h^i_c \right\|^2 \leqslant E_H. 
\end{aligned}
\end{equation}
}
Similarly, the local objective of k-th client with a set of classes $C_k$ can be described as:
{\small
\begin{equation}\label{eq:l_objective}
\begin{aligned}
\min _{\gH_k,\W^L_k} \frac{1}{n_k} \sum_{c\in C_k} \sum_{i=1}^{n_{k,c}}\mathcal{L}_{CE}\left(h^i_{k,c}, \W_k^L\right), \text { s.t. } \left\|\W_{k,c}^L\right\|^2 \leqslant E_W, \ \left\|h_{k,c}^i \right\|^2 \leqslant E_H. 
\end{aligned}
\end{equation}
}

\begin{wrapfigure}{r}[0cm]{0pt}
\includegraphics[width=0.36\textwidth]
{ 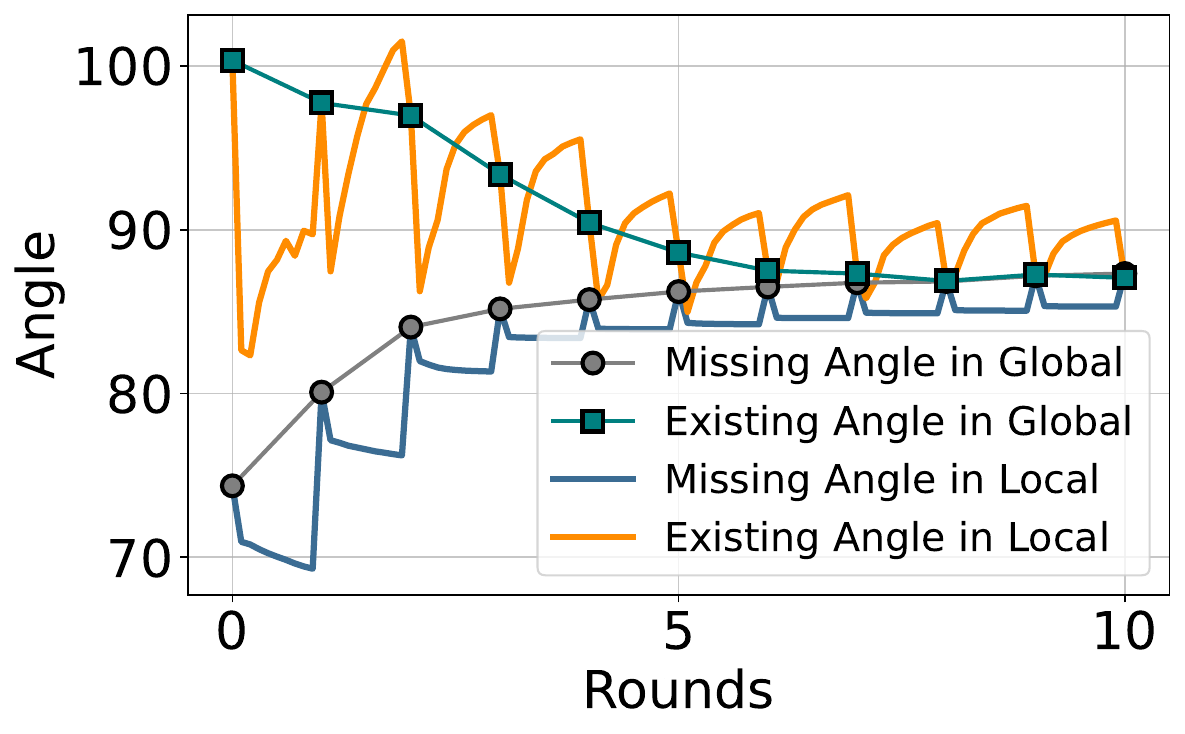} 
\caption{Averaged angles of classifier vectors between locally existing classes~(existing angle) and between locally missing classes~(missing angle) on CIFAR10~(Dir~($\beta=0.1$)) in local client and aggregated in global server (local epoch is 10). 
In global, ``existing'' angle and ``missing'' angle converge to similar values while in the local, ``existing'' angle expands but ``missing'' angle shrinks.}\label{fig:motivation}
\vspace{-20pt}
\end{wrapfigure}

When PCDD exists~($C_k \neq C$), we can see the contradiction between local and global objectives, which respectively forms two structures, shown in Figure~\ref{fig:method}(a) and Figure~\ref{fig:method}(b). After aggregated in server or constrained by some FL methods, the structure in the local is restricted to meet the global structure, causing space waste for personal tasks shown in Figure~\ref{fig:intro}(d).

\textbf{Motivation.}~~~To verify the contradiction and related feature and classifier structures, we split CIFAR10 into 10 clients and perform FedAvg on it with Dirichlet Distribution~(Dir~($\beta=0.1$)). 
As illustrated in Figure~\ref{fig:motivation}, the angle difference between existing classes and between missing classes becomes smaller and converges to a similar value in the global model. However, in the local training, angles between existing classes become larger while angles between missing classes become smaller, which indicates the contradiction. With this observation, to bridge the gap between Eq~\eqref{eq:l_objective} and Eq~\eqref{eq:g_objective} under PCDD, we need to construct the symmetrical and uniform classifier angles for all classes while encouraging local clients to expand existing classes' feature space. Therefore, we propose our method \textbf{FedGELA} that classifier can be \textbf{Globally} fixed as \textbf{ETF} but \textbf{Locally} \textbf{Adapted} based on the local distribution matrix to utilize the wasted space for the existing classes. 

\subsection{FedGELA}\label{sec:fedgela}
\textbf{Global ETF. }~~~Given the global aim of achieving an unbiased classifier that treats all classes equally and provides them with the same discrimination and classifier angles, we curate the global model's classifier as a randomly initialized simplex ETF with scaling $\sqrt{E_w}$ at the start of federated training:
$$
\W^{L}=\sqrt{E_w}\M. 
$$
Then the ETF is distributed to all clients to replace their local classifiers. In Theorem~\ref{thm:global}, we prove in federated training under some basic assumptions, by fixing the classifier as a randomly simplex ETF with scaling $\sqrt{E_W}$ and constraints $E_H$ on the last layer features, features output by last layer of backbone and their within class means will converge to the ETF similar to Eq~\eqref{eq:etf}, which meets the requirement of global tasks. 

\textbf{Local Adaptation. }~~~However, when PCDD exists in the local clients, naively combining ETF with FL does not meet the requirement of P-FL as analyzed in Eq~\eqref{eq:g_objective} and Eq~\eqref{eq:l_objective}. To utilize the wasted space for locally missing classes, in the training stage, we curate the length of ETF received from the server based on the local distribution as below:
{\small
\begin{equation}\label{eq:adjust}
\W_k^L={\PHI_k} \W^L=\PHI_k\sqrt{E_w}\M,
\end{equation}}
where $\PHI_{k}$ is the distribution matrix of k-th client. Regarding the selection of $\boldsymbol{\Phi}_k$, it should satisfy a basic rule for federated learning, wherein the aggregation of local classifiers aligns with the global classifier, thereby ensuring the validity of theoretical analyses from both global and local perspectives. Moreover, it is highly preferable for the selection process to avoid introducing any additional privacy leakage risks. To meet the requirement that averaged classifier should be standard ETF: $\W^L=\sum_{k=1}^{N}p_k \W^L_k$ in the globally balanced case, its row vectors are all one's vector multiple statistical values of personal distribution:$(\PHI_{k}^T)_c= \frac{n_{k,c}}{n_k\gamma} \mathbf{1}$ ($\gamma$ is a constant, and $n_{k,c}$ and $n_k$ are the c-th class sample number and total sample number of the k-th client) respectively. We set $\gamma$ to $\frac{1}{|C|}$. Finally, the local objective from Eq.~\eqref{eq:l_objective} is adapted as:
{\small
\begin{equation}\label{eq:loss}
\begin{aligned}
\min_{\gH_k} \quad & \frac{1}{n_k} \sum_{c=1}^C \sum_{i=1}^{n_{k,c}}-\log \frac{ \exp (  \PHI_{k,c}{\W^L_{c}}^T h_{k,c}^i)}{\sum_{c' \in C_k}  \exp ( \PHI_{k,c'}{\W^L_{c'}}^Th_{k,c}^i)},\\
\text { s.t. } \quad &\|h^i\|^2 \leqslant E_H, \forall 1 \leqslant i \leqslant n_k.
\end{aligned}
\end{equation}}

\begin{wrapfigure}{r}[0cm]{0pt}
\begin{minipage}{.5\linewidth}
\vspace{-25pt}
\begin{algorithm}[H]
    {\small
    \caption{FedGELA}
    \label{alg:fedgela}
    \textbf{Input}:$(N, K, n_k, c_k, \gH^0, \M, E_W, E_H, T, \eta, E)$
    \begin{algorithmic}
    \vspace{-10pt}
    \STATE \textbf{Parallelly for all clients: } \colorbox{green!20}
    {$\W^L_k\leftarrow \PHI_k\sqrt{E_W}\M$.}
    \FOR{$t = 0,1,\dots,T-1$}
        \STATE \colorbox{gray!20}{$\rhd$ on the server side}
        \STATE $\gH^t \leftarrow \sum_{k=1}^{K}{p^t_kH^{t-1}_{k}}$.
        \STATE sample K clients from all N clients.
        \STATE \colorbox{gray!20}{$\rhd$ on the client side}
        \ONCLIENT{$\forall{k \in K}$}
            \STATE receive $\gH^t$ from server, $\gH_k^t \leftarrow \gH^t$.
            \FOR{$\tau=0,1,...,E-1$} 
                \STATE sample a mini-batch $b_k^{tE+\tau}$ in local data.
            \STATE\colorbox{green!20}{$H_k^{t} \leftarrow H_{k}^t - \eta \nabla F_k(b_{k}^t,\PHI_k W_g^L;\gH_{k}^t)$} 
            \ENDFOR
            \STATE submit $\gH^{t}_{k}$ to the server.
        \ENDON
    \ENDFOR
    
    \end{algorithmic}
\colorbox{green!20}{\textbf{Output}:$(\gH^T,W_g^L)$ and $(\gH^T_k,\PHI_{k}W_g^L) $.}
}
\end{algorithm}
\end{minipage}
\end{wrapfigure}
\textbf{Total Framework. }~~~After introducing two key parts of FedGELA (Global ETF and Local Adaptation), we describe the total framework of FedGELA. As illustrated and highlighted in Algorithm~\ref{alg:fedgela} (refer to Appendix~\ref{ap:D} for the workflow figure), at the initializing stage, the server randomly generates an ETF as the global classifier and sends it to all clients while local clients adjust it based on the personal distribution matrix as Eq~\eqref{eq:adjust}. At the training stage, local clients receive global backbones and train with adapted ETF in parallel. After $E$ epochs, all clients submit personal backbones to the server. In the server, personal backbones are received and aggregated to a generic backbone, which is broadcast to all clients participating in the next round. At the inference stage, on the client side, we obtain a generic backbone with standard ETF to handle the world data while on the client side, a personal backbone with adapted ETF to handle the personal data.

\vspace{-5pt}
\section{Theoretical Analysis}\label{sec:analysis}
\vspace{-10pt}

In this part, we first primarily introduce some notations and basic assumptions in Sec.~\ref{sec:assumption} and then present the convergence guarantees of both local models and the global model under the PCDD with the proper empirical justification and discussion in Sec.~\ref{sec:convergence}. (Please refer to Appendix~\ref{ap:B} for entire proofs and Appendix~\ref{ap:D} for details on justification experiments.)

\begin{figure}[ht]
\centering  
\subfigure[Global convergence verification.]{
\includegraphics[width=0.315\textwidth]{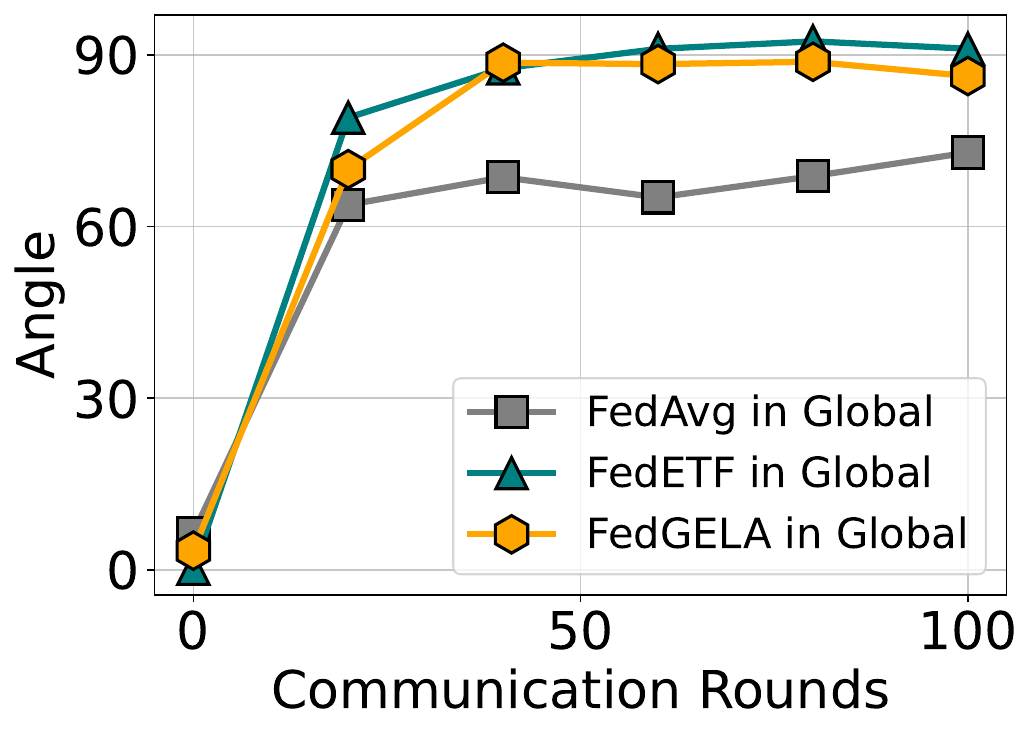}} \hspace{10pt}
\subfigure[Local convergence verification.]{
\includegraphics[width=0.315\textwidth]{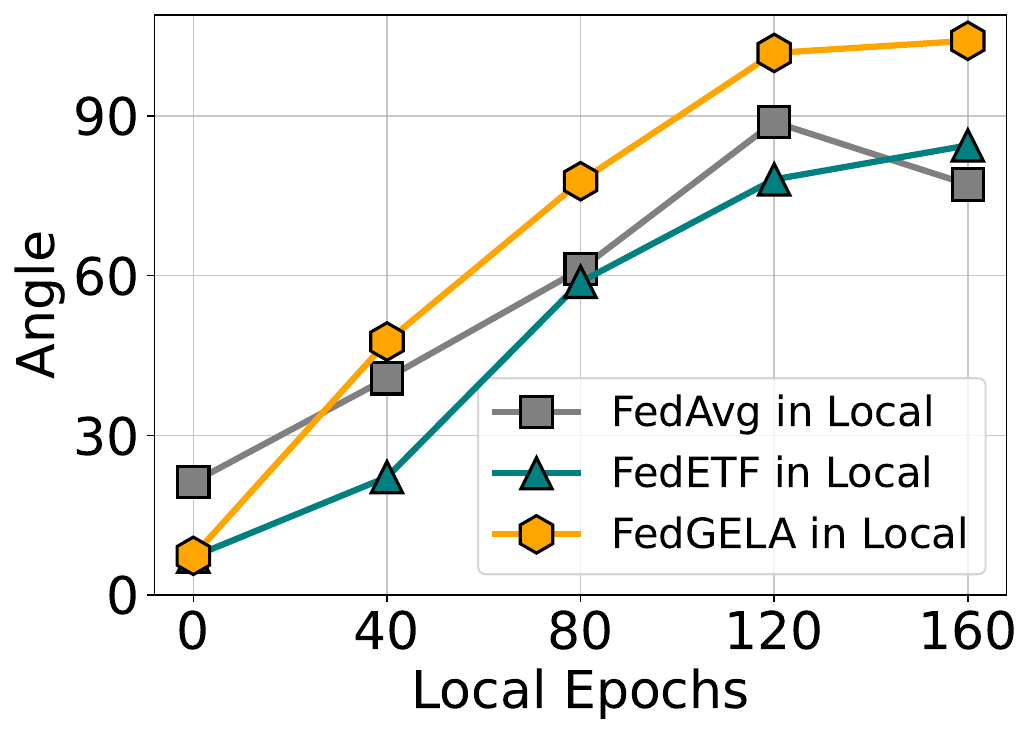}}\hspace{15pt}
\subfigure[Adapted structure.]{
\includegraphics[width=0.19\textwidth]{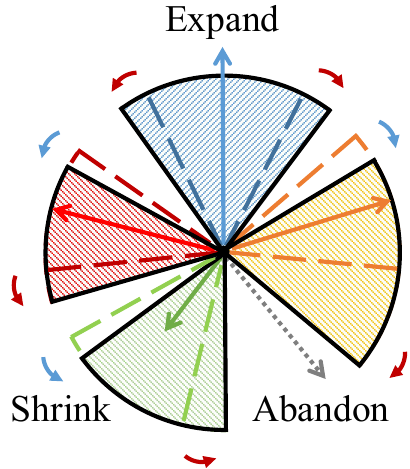}}
\vspace{-7pt}
\caption{Illustration of local and global convergence verification together with the effect of $\PHI$. (a) and (b) are the results of averaged angle between all class means and between locally existing class means in FedAvg, FedGE, and FedGELA on CIFAR10 under 50 clients and Dir~($\beta=0.2$). (c) is the illustration of how local adaptation utilizes the wasted space of missing classes for existing classes.}
\label{fig:method}
\vspace{-12pt}
\end{figure}

\subsection{Notations}\label{sec:assumption}
We use $t$ and $T$ to denote a curtain round and pre-defined maximum round after aggregation in federated training, $tE$ to denote the state that just finishing local training before aggregation in round 
$t$, and $tE+\tau$ to denote $\tau$-th local iteration in round $t$ and $0\leq \tau \leq E-1$. The convergence follows some common assumptions in previous FL studies and helpful math results~\citep{fedproto,scaffold,them_noniid,theo_jianyuwang,stich,assumption1,assumption2,assumption3,assumption4,lemma4} including smoothness, convexity on loss function~$F_1, F_2,\cdots, F_N$ of all clients, bounded norm and variance of stochastic gradients on their gradient functions~$\nabla F_1, \nabla F_2,\cdots, \nabla F_N$ and heterogeneity~$\Gamma_1$ reflected as the distance between local optimum $\W_k^*$ and global optimum $\W^*$. Please refer to the concrete descriptions of those assumptions in Appendix~\ref{ap:B}. Besides, in Appendix~\ref{ap:B}, we additionally provide a convergence guarantee without a bounded norm of stochastic gradients, as some existing works~\citep{contradiction1,contradiction2} point out the contradiction to the strongly convex. 

\subsection{Convergence analysis}\label{sec:convergence}
Here we provide the global and local convergence guarantee of our FedGELA compared with FedAvg and FedGE~(FedAvg with only the Globally Fixed ETF) in Theorem~\ref{thm:global} and Theorem~\ref{thm:local}. To better explain the effectiveness of our FedGELA in local and global tasks, we 
record the averaged angle between all class means in global and existing class means in local as shown in Figure~\ref{fig:method}(a) and Figure~\ref{fig:method}(b). Please refer to Appendix~\ref{ap:B} for details on the proof and justification of theorems.

\begin{theorem}[Global Convergence] \label{thm:global}
If $F_1,...,F_N$ are all L-smooth, $\mu$-strongly convex, and the variance and norm of $\nabla F_1,...,\nabla F_N$ are bounded by $\sigma$ and $G$. Choose $\kappa=L / \mu$ and $\gamma=\max\{8\kappa, E\}$, for all classes $c$ and sample $i$, expected global representation by cross-entropy loss will converge to:
{\small
$$
\mathbb{E} \left[\log \frac{(\W^{L,*})^T h^{i,*}_{c}}{(\W_g^L)^T h^{i}_{c}} \right ] \leq \frac{\kappa}{\gamma+T-1} \left (  \frac{2B}{\mu} + \frac{\mu \gamma} {2} \mathbb{E} || \W^1 - \W^* ||^2 \right),
$$
}
where in FedGELA, $B = \sum_{k=1}^N (p_k^2 \sigma^2 + p_k ||\PHI_k\W^L - \W^L|| ) + 6 L \Gamma_1 + 8(E-1)^2G^2$. Since $\W^{L}=\W^{L,*}$ and $(\W^{L,*})^T h^{i,*}_{c_i} \geq \mathbb{E}[(\W^L)^T h^{i}_{c_i}]$, $h^{i}_{c_i}$ will converge to $h^{i,*}_{c_i}$.
\end{theorem}
\vspace{-5pt}
In Theorem~\ref{thm:global}, the variable $B$ represents the impact of algorithmic convergence ($p_k^2 \sigma^2$), non-iid data distribution ($6 L \Gamma_1$), and stochastic optimization ($8(E-1)^2G^2$). The only difference between FedAvg, FedGE, and our FedGELA lies in the value of $B$ while others are kept the same. FedGE and FedGELA have a smaller $G$ compared to FedAvg because they employ a fixed ETF classifier that is predefined as optimal. FedGELA introduces a minor additional overhead ($p_k || \PHI_k\W^L - \W^L||$) on the global convergence of FedGE due to the incorporation of local adaptation to ETFs.
The cost might be negligible, as $\sigma$, $G$, and $\Gamma_1$ are defined on the whole model weights while $p_k \|\PHI_k\W^L-\W^L\|$ is defined on the classifier. To verify this, we conduct experiments in Figure~\ref{fig:method}(a), and as can be seen, FedGE and FedGELA have similar quicker speeds and larger classification angles than FedAvg.

\vspace{5pt}

\begin{theorem}[Local Convergence]\label{thm:local}
If $F_1,...,F_N$ are L-smooth, variance and norm of their gradients are bounded by $\sigma$ and $G$, and the heterogeneity is bounded by $\Gamma_1$, clients' expected local loss satisfies:
{\small$$
\mathbb{E}[F_k^{(t+1) E}] \leqslant F_k^{t E}+\frac{L E \eta_t^2}{2} \sigma^2 +\Gamma_1-A,
$$}
where in FedGELA, $A=(\eta_t-\frac{L}{2} \eta_t^2) E G^2-L\left\|\PHI_k\W^L-\W^L\right\|$, which means if $A-\frac{G^4}{LE(G^2+\sigma^2)}\leq 0$, there exist learning rate $\eta_t$ making the expected local loss decreasing and converging.
\end{theorem}

In Theorem~\ref{thm:local}, only ``A'' is different on the convergence among FedAvg, FedGE, and FedGELA. 
Fixing the classifier as ETF and adapting the local classifier will introduce smaller G and additional cost of $L\left\|\PHI_k\W^L-\W^L\right\|$ respectively, which might limit the speed of local convergence. However, FedGELA might reach better local optimal by adapting the feature structure. As illustrated in Figure~\ref{fig:method}~(c), the adapted structure expands the decision boundaries of existing major classes and better utilizes the feature space wasted by missing classes.

To verify this, in Figure~\ref{fig:method}(b), we record the averaged angles between the existing class means during the local training. It can be seen that FedGELA converges to a much larger angle than both FedAvg and FedGE, which suits our expectations. More angle results can be seen in Figure~\ref{fig:more}.

\section{Experiments}\label{sec:experiment}
\subsection{Experimental Setup}
\textbf{Datasets.}~~~We adopt three popular benchmark datasets SVHN~\cite{svhn}, CIFAR10/100~\cite{cifar10} in federated learning. As for data splitting, we utilize Dirichlet Distribution~(Dir~($\beta$), $\beta=\{10000, 0.5$, $0.2$, $0.1\}$) to simulate the situations of independently identical distribution and different levels of PCDD. Besides, one standard real-world PCDD dataset, Fed-ISIC2019~\cite{terrail2022flamby, codella2018skin, combalia2019bcn20000, tschandl2018ham10000} is used, and we follow the setting in the Flamby benchmark~\cite{terrail2022flamby}. Please refer to Appendix~\ref{ap:C} for more details. 

\textbf{Metrics.}~~~Denote PA as the personal accuracy, which is the mean of the accuracy computed on each client test dataset, and GA as the generic accuracy on global test dataset~(mixed clients' test datasets). Since there is no global model in P-FL methods, we calculate GA of them as the averaged accuracy of all best local models on global test dataset, which is the same as FedRod~\citep{fedrod}. 

Regarding PA, we record the best results of personal models for P-FL methods while for G-FL methods we fine-tune the best global model in 10 epochs and record the averaged accuracy on all client test datasets. For FedRod and FedGELA, we can directly record the GA and PA (without fine-tuning) during training. 

\textbf{Implementation.}~~~We compare FedGELA with FedAvg, FedRod~\citep{fedrod}, multiple state-of-the-art methods in G-FL (FedRS~\citep{fedrs}, MOON~\citep{moon}, FedProx~\citep{fedprox}, FedGen~\citep{fedgen} and FedLC~\citep{fedlc}) and in P-FL (FedRep~\citep{fedrep}, FedProto~\citep{fedproto} and FedBABU~\citep{fedbabu}). For SVHN, CIFAR10, and CIFAR100, we adopt a commonly used ResNet18~\citep{moon,huang,zhou,fedskip,fedlc} with one FC layer as the backbone, followed by a layer of classifier. FedGELA replaces the classifier as a simple ETF. We use SGD with learning rate 0.01, weight decay $10^{-4}$, and momentum 0.9. The batch size is set as 100 and the local updates are set as 10 epochs for all approaches. 

As for method-specific hyper-parameters like the proximal term in FedProx, we tune it carefully. In our method, there are $E_W$ and $E_H$ need to set, we normalize features with length 1~($E_H=1$) and only tune the length scaling of classifier~($E_W$). All methods are implemented by PyTorch~\citep{pytorch} with NVIDIA GeForce RTX 3090. See detailed information in Appendix~\ref{ap:C}.

\begin{table}[th]
\setlength{\tabcolsep}{3.75pt}
\centering
\caption{Personal and generic performance on SVHN, CIFAR10, and CIFAR100. We use Dir~($\beta=0.5$) for medium heterogeneity and Dir~($\beta=0.1$) or Dir~($\beta=0.2$) for high-level heterogeneity. To verify the straggler situation, we split all datasets into 10 or 50 clients for full participation or partial participation, and in each round, 10 clients are selected in the federated training.}
\vspace{-2pt}
\label{tb:cifar}
\small
\centering
\scalebox{0.90}{
\begin{tabular}{c|c|cccccc|cccccc}
\toprule[2pt]
\multirow{3}{*}[-4pt]{ Dataset } & Method 
 & \multicolumn{6}{c|}{ Full Participation $(10,10)$} & \multicolumn{6}{c}{ Partial Participation $(50,10)$} \\
\cmidrule{2-14} & \#Partition & \multicolumn{2}{c}{ IID } & \multicolumn{2}{c}{$\beta=0.5$} & \multicolumn{2}{c|}{$\beta=0.1$} & \multicolumn{2}{c}{ IID } & \multicolumn{2}{c}{$\beta=0.5$} & \multicolumn{2}{c}{$\beta=0.2$} \\
\cmidrule{2-14} & \#Metric& PA & GA & PA & GA & PA & GA & PA & GA & PA & GA & PA & GA \\
\midrule 
\multirow{10}{*}[0pt]{SVHN} & FedAvg & $93.01$ & $92.61$ & $93.95$ & $91.24$ & $98.10$ & $75.24$& $91.44$ & $91.29$ & $92.70$ & $89.29$ & $95.31$ & $84.70$  \\
 & FedProx & $93.12$ & $93.12$ & $93.71$ & $92.15$ & $97.98$ & $75.13$& $91.67$ & $91.66$ & $92.71$ & $89.98$ & $95.13$ & $85.68$  \\
 & MOON & $93.16$ & $93.16$ & $92.98$ & $92.46$ & $98.06$ & $76.21$ & $93.49$ & $91.41$ & $91.86$ & $90.20$ & $95.78$ & $86.22$ \\
  & FedRS & $93.29$ & $93.21$ & $93.92$ & $92.33$ & $98.04$ & $\underline{76.26}$ & $91.63$ & $91.59$ & $93.51$ & $\underline{91.70}$ & $\underline{96.20}$ & $\underline{87.78}$ \\
  & FedGen & $\underline{94.02}$ & $\underline{93.99}$ & $94.47$ & $\underline{92.66}$ & $98.22$ & $76.51$ & $91.47$ & $91.33$ & $93.67$ & $91.35$ & $95.77$ & $87.59$ \\
 & FedLC & $93.29$ & $\underline{93.28}$ & $94.76$ & $91.20$ & $98.24$ & $76.17$ & $91.69$ & $\underline{91.67}$ & $92.73$ & $91.02$ & $95.20$ & $86.92$ \\
 & FedRep & $93.01$ & $92.61$ & $94.77$ & $91.24$ & $97.87$ & $68.52$ & $91.77$ & $89.20$ & $93.14$ & $80.94$ & $95.38$ & $67.77$ \\
 & FedProto & $93.21$ & $91.68$ & $94.48$ & $85.85$ & $\underline{98.26}$ & $56.49$ & $90.23$ & $87.27$ & $93.28$ & $76.59$ & $95.62$ & $54.92$ \\
 & FedBABU & $93.26$ & $93.08$ & $95.20$ & $92.04$ & $98.16$ & $75.52$ & $\underline{93.69}$ & $91.05$ & $93.54$ & $90.49$ & $95.70$ & $84.42$ \\
 & FedRod & $93.50$ & $93.22$ & $\underline{95.47}$ & $92.09$ & $98.06$ & $76.24$ & $92.04$ & $91.65$ & $\underline{93.96}$ & $91.20$ & $95.68$ & $86.98$ \\
& \textbf{FedGELA} & $\textbf{94.84}$ & $\textbf{94.66}$ & $\textbf{96.27}$ & $\textbf{93.66}$ & $\textbf{98.52}$ & $\textbf{78.88}$ & $\textbf{94.68}$ & $\textbf{93.59}$ & $\textbf{95.54}$ & $\textbf{93.29}$ & $\textbf{96.85}$ & $\textbf{89.58}$ \\
\midrule 
\multirow{10}{*}{CIFAR10} & FedAvg & $73.17$ & $72.8$ & $81.67$ & $67.28$ & $92.66$ & $54.57$ & $66.88$ & $66.64$ & $70.64$ & $61.81$ & $80.04$ & $49.13$ \\
 & FedProx & $73.69$ & $73.69$ & $81.95$ & $67.53$ & $92.94$ & $56.13$ & $67.67$ & $67.27$ & $73.62$ & $60.80$ & $80.66$ & $50.82$ \\
 & MOON & $73.29$ & $73.29$ & $82.27$ & $68.34$ & $92.90$ & $55.61$ & $67.58$ & $67.58$ & $74.64$ & $61.81$ & $83.42$ & $52.19$ \\
  & FedRS & $73.56$ & $72.94$ & $81.59$ & $68.10$ & $92.57$ & $58.19$ & $66.76$ & $66.52$ & $72.21$ & $58.95$ & $81.11$ & $51.66$ \\
 & FedLC & $73.05$ & $73.00$ & $81.99$ & $67.97$ & $92.48$ & $57.02 $&$ 67.46$ & $67.13$ & $72.57$ & $61.31$ & $82.14$ & $55.15$ \\
   & FedGen & $73.72$ & $73.49$ & $82.22$ & $69.33$ & $92.79$ & $58.04$ & $68.74$ & $68.02$ & $75.52$ & $62.44$ & $81.07$ & $53.46$ \\
 & FedRep & $73.42$ & $73.23$ & $\underline{83.30}$ & $47.96$ & $92.92$ & $38.32$ & $67.85$ & $67.74$ & $77.28$ & $42.64$ & $84.52$ & $33.22$ \\
 & FedProto & $67.06$ & $66.74$ & $81.03$ & $46.99$ & $\underline{93.17}$ & $32.13$ & $61.85$ & $52.76$ & $72.89$ & $37.47$ & $81.73$ & $26.07$ \\
 & FedBABU & $73.86$ & $72.30$ & $81.40$ & $65.03$ & $92.94$ & $53.65$ & $66.99$ & $64.90$ & $77.59$ & $58.17$ & $82.92$ & $49.90$ \\
 & FedRod & $\underline{74.24}$ & $\underline{73.76}$ & $82.34$ & $\underline{70.74}$ & $92.27$ & $\underline{58.86}$ & $\underline{70.09}$ & $\underline{70.04}$ & $\underline{78.23}$ & $\underline{64.13}$ & $\underline{84.63}$ & $\underline{58.86}$ \\
&  \textbf{FedGELA} & $\textbf{75.02}$ & $\textbf{74.07}$ & $\textbf{84.52}$ & $\textbf{72.73}$ & $\textbf{94.28}$ & $\textbf{61.57}$ & $\textbf{72.33}$ & $\textbf{72.04}$ & $\textbf{80.96}$& $\textbf{65.08}$ & $\textbf{86.55}$ & $\textbf{60.52}$ \\
\midrule 
\multirow{10}{*}{CIFAR100} & FedAvg & $65.27$ & $65.27$ & $65.59$ & $63.96$ & $76.43$ & $59.17 $& $55.16$ & $55.29$ & $55.36$ & $54.15$ & $58.85$ & $53.39$ \\
 & FedProx & $65.71$ & $65.71$ & $65.31$ & $64.18$ & $75.95$ & $59.93$ & $56.86$ & $56.89$ & $56.89$ & $56.08$ & $59.27$ & $55.25$ \\
 & MOON & $65.33$ & $65.33$ & $65.23$ & $64.79$ & $75.45$& $60.12$ &$56.91$& $56.86$ & $56.72$ & $56.14$ & $59.51$ & $55.53$\\
 & FedRS & $65.18$ & $65.64$ & $66.48$ & $64.62$ & $76.86$ & $60.74$ & $56.51$ & $55.91$ & $56.45$ & $56.34$ & $61.92$ & $55.99$ \\
  & FedGen & $65.74$ & $65.75$ & $66.72$ & $64.33$ & $76.92$ & $60.43$ & $56.77$ & $56.74$ & $57.43$ & $56.27$ & $60.09$ & $55.27$ \\
 & FedLC & $65.83$ & $65.84$ & $65.91$ & $65.02$ & $75.67$ & $60.07$ & $56.87$ & $56.04$ & $56.56$ & $56.28$ & $60.89$ & $\underline{55.95}$ \\
 & FedRep & $61.21$ & $59.21$ & $67.87$ & $52.51$ & $77.81$ & $42.77$ & $53.41$ & $51.44$ & $55.60$ & $48.67$ & $67.70$ & $33.10$ \\
 & FedProto & $56.56$ & $56.26$ & $66.08$ & $46.88$ & $77.68$ & $37.63$ & $52.41$ & $50.04$ & $54.05$ & $42.88$ & $63.22$ & $28.74$ \\
 & FedBABU & $65.63$ & $65.28$ & $71.30$ & $64.54$ & $80.33$ & $60.99$ & $56.91$ & $54.57$ & $60.14$ & $54.40$ & $68.44$ & $54.24$ \\
 & FedRod & $\underline{66.17}$ & $\underline{66.17}$ & $\underline{72.05}$ & $\underline{65.19}$ & $\underline{80.46}$ & $\underline{61.01}$ & $\underline{57.76}$ & $\underline{57.01}$ & $\underline{63.90}$ & $\underline{56.53}$ & $\underline{72.37}$ & $54.67$ \\
 & \textbf{FedGELA} & $\textbf{67.28}$ & $\textbf{68.07}$ & $\textbf{72.61}$ & $\textbf{66.94}$ & $\textbf{82.79}$ & $\textbf{63.13}$ & $\textbf{61.70}$ & $\textbf{59.29}$ & $\textbf{64.37}$ & $\textbf{58.60}$ & $\textbf{72.93}$ & $\textbf{58.53}$ \\
\bottomrule[2pt]
\end{tabular}
}
\vspace{-10pt}
\end{table}

\begin{table}[th]
\setlength{\tabcolsep}{1pt}
\caption{Personal and generic performance on a real federated application Fed-ISIC2019. More results of other realworld dataset are shown in the Appendix.}
\small
\centering
\scalebox{0.79}{
\begin{tabular}{c | cc c c cc ccccc}
\toprule[2pt]  Method & FedAvg& FedProx&MOON&FedRS&FedGen&FedLC& FedRep&FedProto&FedBABU&FedRod&\textbf{FedGELA} \\ \midrule
\text { PA } & $77.27_{\pm 0.19}$&$77.91_{\pm 0.16}$ & $77.94_{\pm 0.17}$& $78.27_{\pm 0.12}$& $78.02_{\pm 0.23}$& $77.58_{\pm 0.19}$ & $76.94_{\pm 0.13}$&$77.80_{\pm 0.17}$ & $\underline{78.91}_{\pm 0.13}$& $78.65_{\pm 0.34}$& $\textbf{79.27}_{\pm 0.19}$\\
\cmidrule{2-12} 
\text { GA }& $73.59_{\pm 0.17}$&$73.69_{\pm 0.26}$ & $73.80_{\pm 0.21}$& $74.60_{\pm 0.15}$& $74.37_{\pm 0.27}$& $74.26_{\pm 0.25}$ & $68.05_{\pm 0.37}$& $66.26_{\pm 0.16}$& $74.06_{\pm 0.31}$& $\underline{74.98}_{\pm 0.21}$& $\textbf{75.85}_{\pm 0.16}$\\
\bottomrule[2pt]
\end{tabular}
}
\label{tab:real}
\vspace{-10pt}
\end{table}

\subsection{Performance of FedGELA}\label{sec:expmain}
In this part, we compare FedGELA with FedAvg, FedRod, three SOTA methods of P-FL~(FedRep, FedProto, and FedBABU), four SOTA methods of G-FL~(FedProx, MOON, FedRS, FedLC and FedGen) on different aspects including the scale of clients, the level of PCDD, straggler situations, and real-world applications. Similar to recent studies~\citep{fedskip,moon,ye2}, we split SVHN, CIFAR10, and CIFAR100 into 10 and 50 clients and each round select 10 clients to join the federated training, denoted as full participation and partial participation~(straggler situation), respectively. With the help of Dirichlet distribution~\citep{diri1}, we verify all methods on IID, Non-IID~($\beta=0.5$), and extreme Non-IID situations~($\beta=0.1$ or $\beta=0.2$). As the decreasing $\beta$, the level of PCDD increases and we show the heat map of data distribution in Appendix~\ref{ap:C}. We set $\beta=0.2$ in partial participation to make sure each client has at least one batch of samples. The training round for SVHN and CIFAR10 is 50 in full participation and 100 in partial participation while for CIFAR100, it is set to 100 and 200. Besides, we also utilize a real federated scenario Fed-ISIC2019 to verify the ability to real-world application.

\textbf{Full participation and partial participation.}~~~As shown in Table~\ref{tb:cifar}, with the decreasing $\beta$ or increasing number of clients, the generic performance of FedAvg and all other methods greatly drops while the personal performance of all methods greatly increases. This means under PCDD and the straggler problem, the performance of generic performance is limited but the personal distribution is easier to capture. As for P-FL methods, they fail in global tasks especially in severe PCDD situations since they do not consider the global convergence during training. As for G-FL methods, the performance is better in generic tasks but limited in personalized tasks, especially in CIFAR100. 
They constrain the model's ability to fit personalized distributions during local training, resulting in improved consistency during global optimization. As can be seen, our FedGELA consistently exceeds all baselines for all settings with averaged performance of 2.42\%, 5.2\% and 5.7\% to FedAvg and 1.35\%, 1.64\% and 1.81\% to the best baseline on the three datasets respectively.

\newcommand{\ablationwid}{0.22\textwidth}
\begin{figure}[t!]
\centering
\subfigure{
\hspace{-5.2pt}
\includegraphics[width=0.22\textwidth]{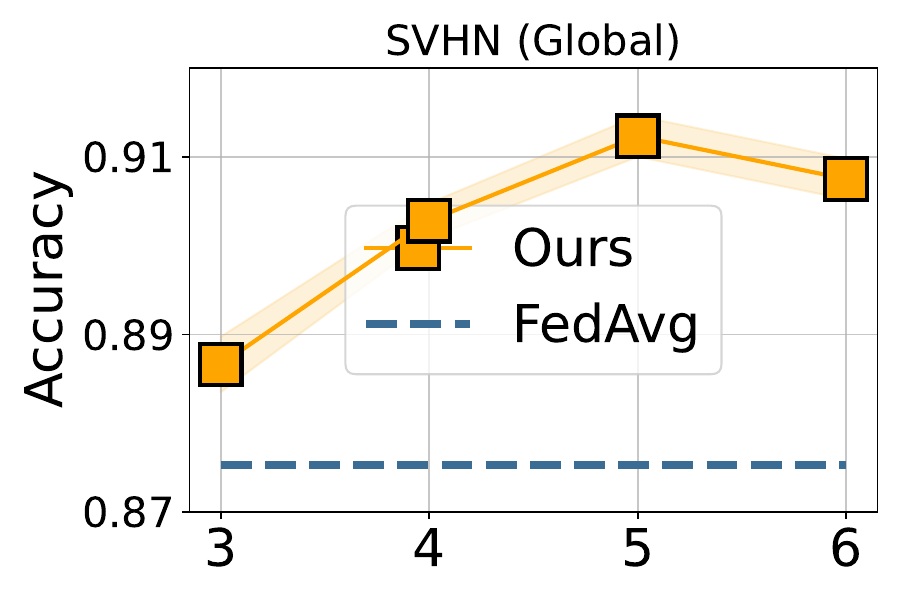}}\vspace{-8pt}
\subfigure{
\includegraphics[width=0.22\textwidth]{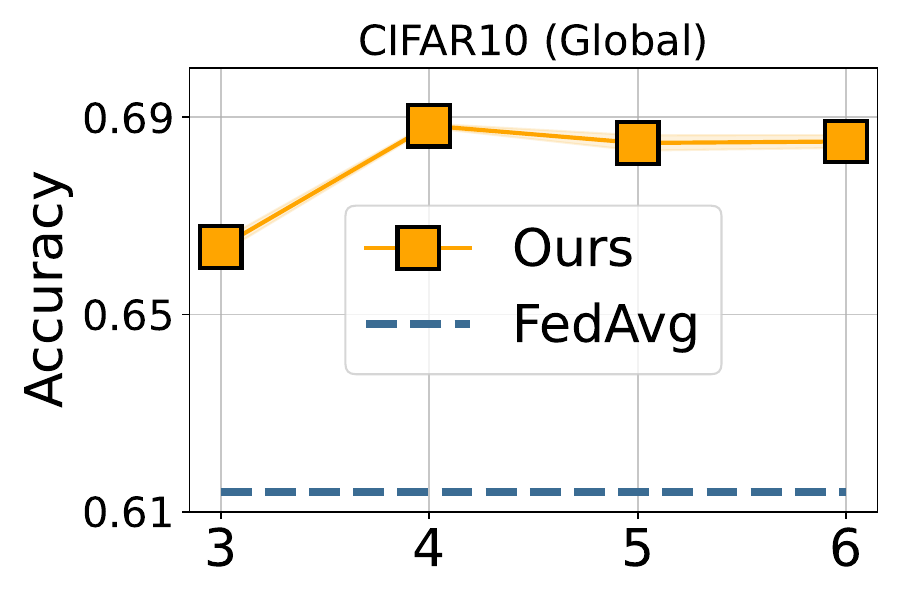}}
\subfigure{
\includegraphics[width=0.22\textwidth]{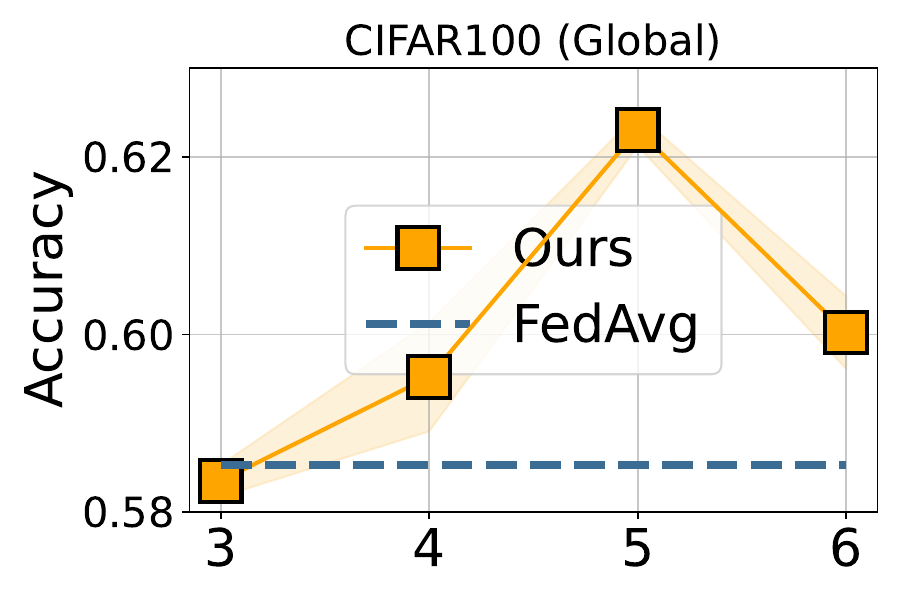}}
\subfigure{
\includegraphics[width=0.22\textwidth]{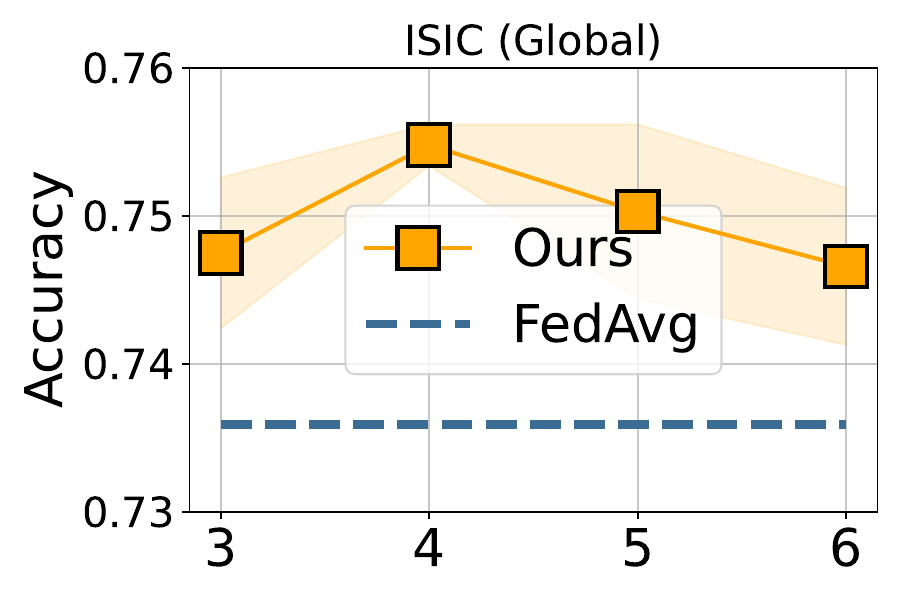}}
\subfigure{
\includegraphics[width=\ablationwid]{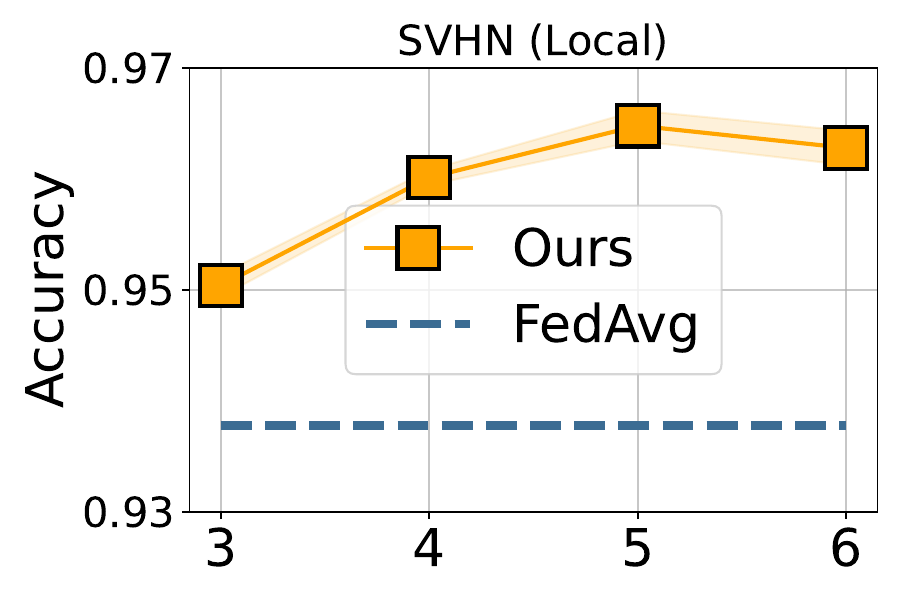}}
\subfigure{
\includegraphics[width=\ablationwid]{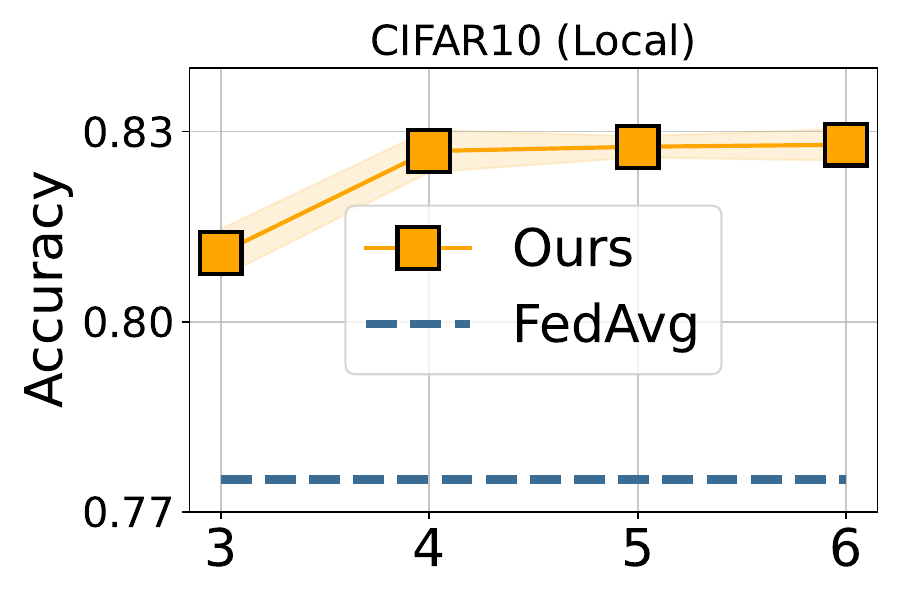}}
\subfigure{
\includegraphics[width=\ablationwid]{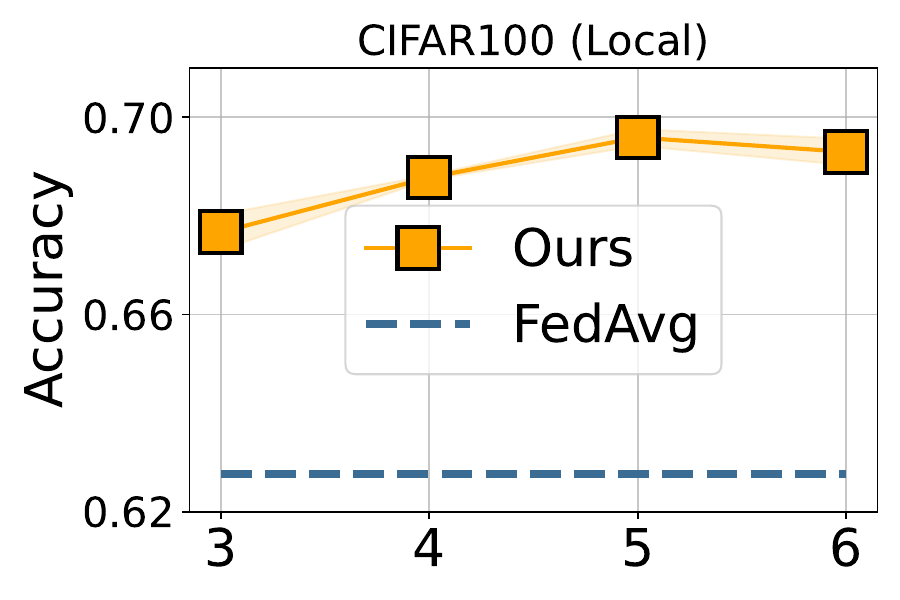}}
\subfigure{
\includegraphics[width=\ablationwid]{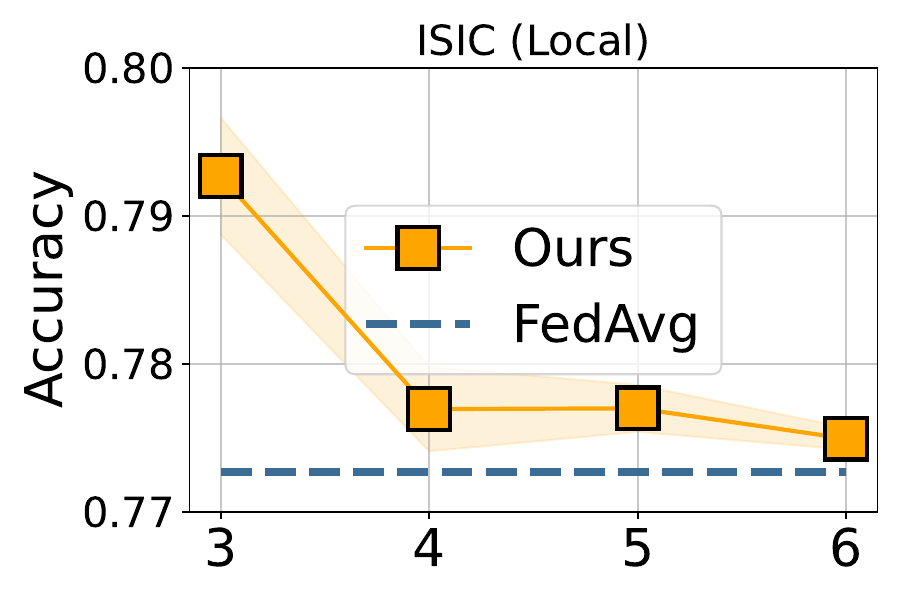}}
\vspace{-9pt}
\caption{Bilateral performance on four datasets by tuning $logE_W$~(x axis) of FedGELA.}
\vspace{-10pt}
\label{fig:ablation}
\end{figure}

\newcommand{\morewid}{0.29\textwidth}
\begin{figure}[ht]
\centering  
\subfigure[Global, Full Parti., $\beta=0.5$]{
\hspace{-2pt}
\includegraphics[width=0.285\textwidth]{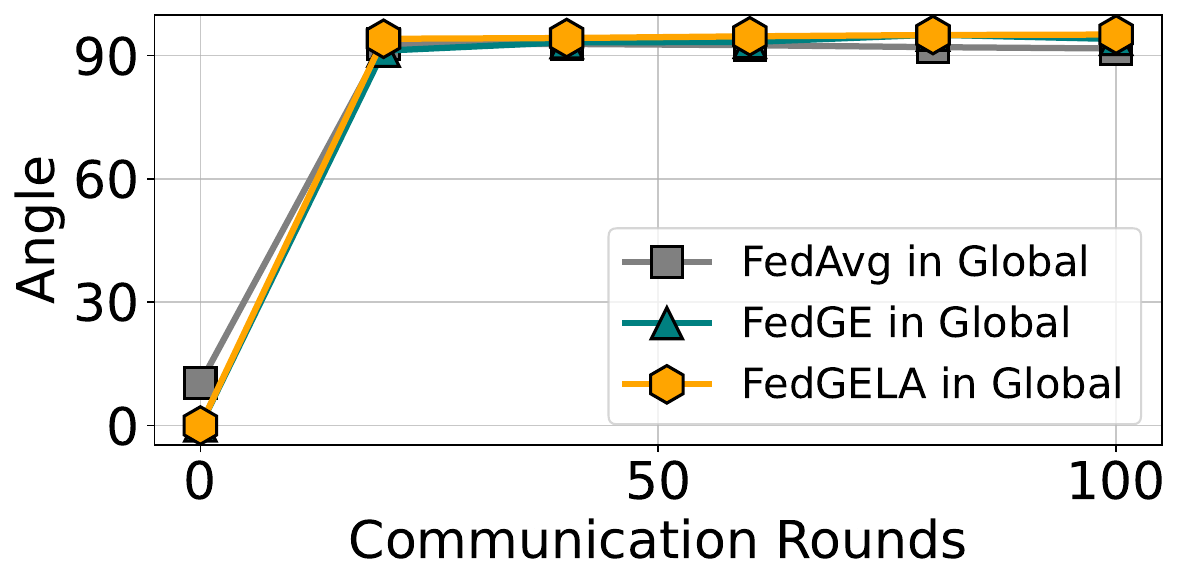}}
\subfigure[Global, Full Parti., $\beta=0.1$]{
\hspace{-1pt}
\includegraphics[width=0.285\textwidth]{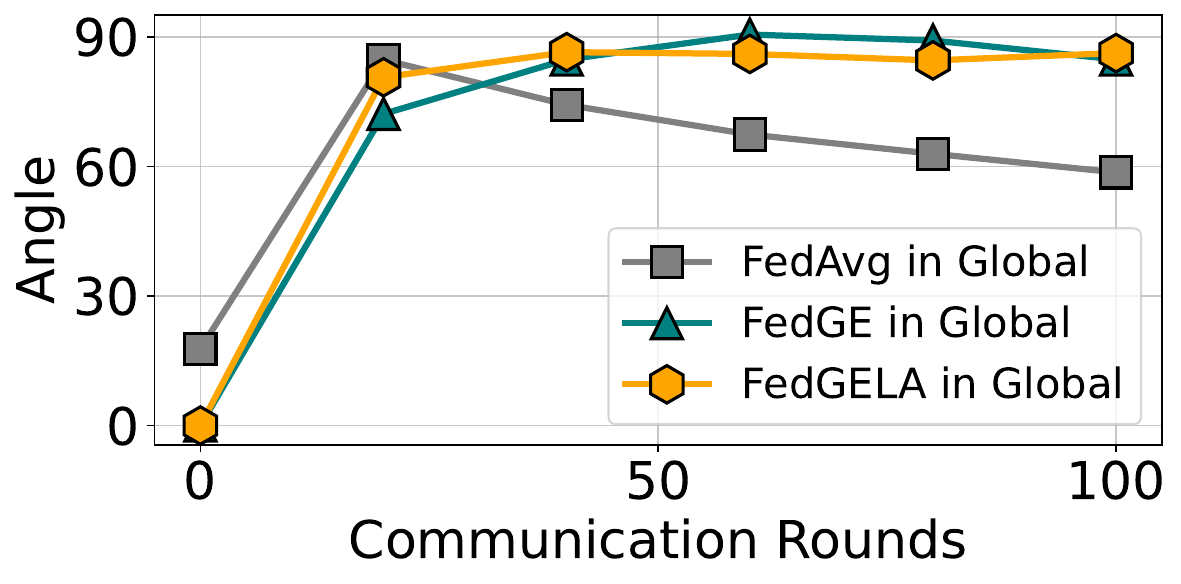}}
\subfigure[Global, Partial Parti., $\beta=0.5$]{
\hspace{1pt}
\includegraphics[width=0.285\textwidth]{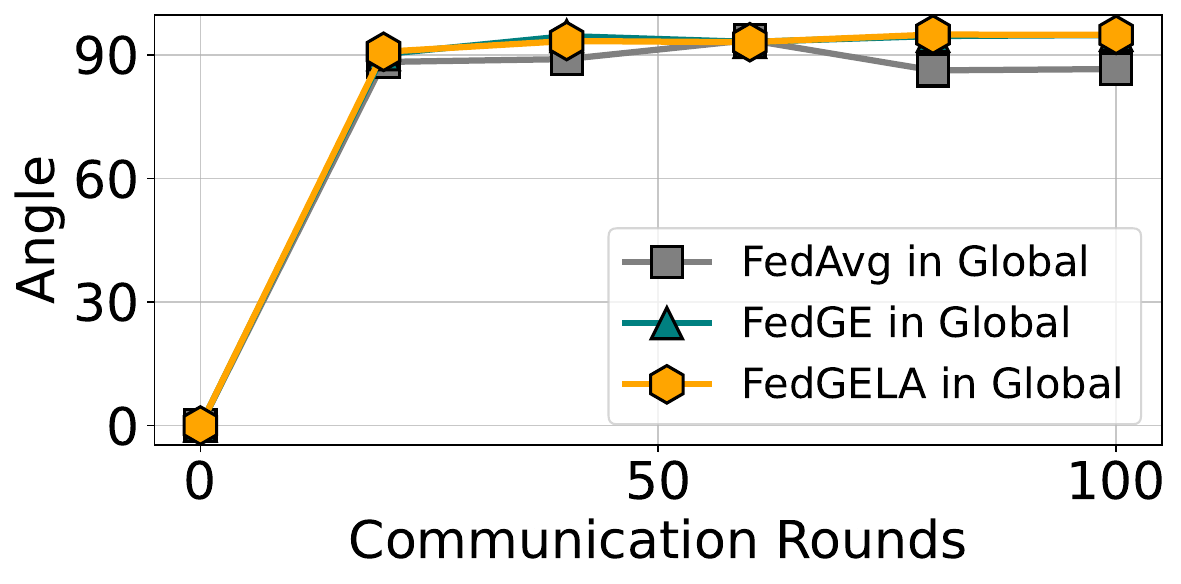}}
\subfigure[Local, Full Parti., $\beta=0.5$]{
\includegraphics[width=\morewid]{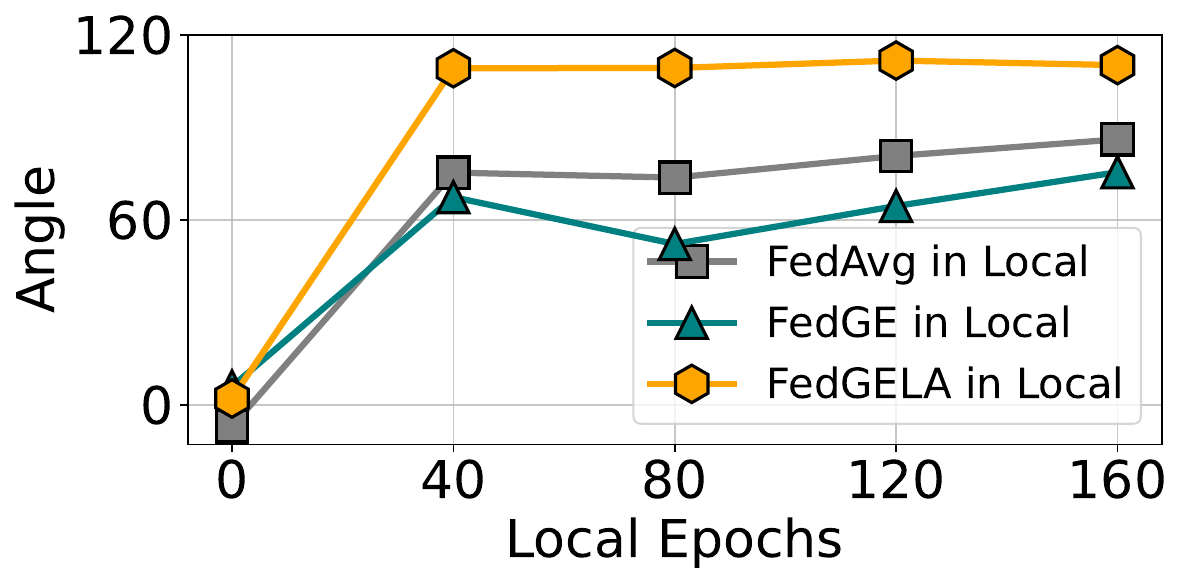}}
\subfigure[Local, Full Parti., $\beta=0.1$]{
\hspace{0pt}
\includegraphics[width=\morewid]{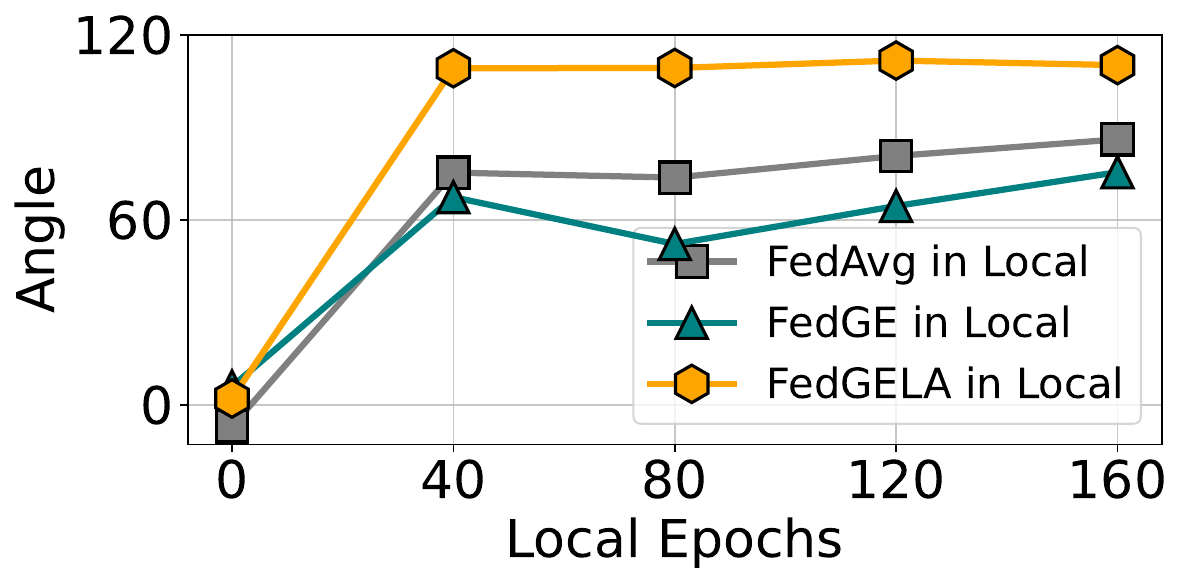}}
\subfigure[Local, Partial Parti., $\beta=0.5$]{
\hspace{-1pt}
\includegraphics[width=\morewid]{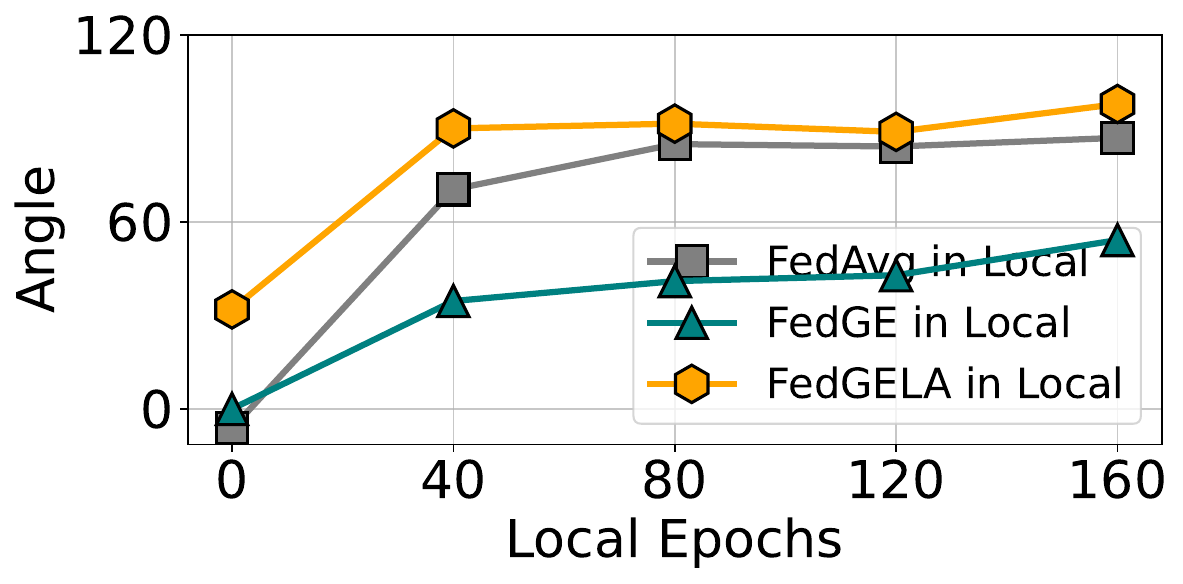}}
\vspace{-5pt}
\caption{Illustration of the averaged angle between locally existing classes and missing classes on the local client and global server of FedAvg, FedGE, and our FedGELA on CIFAR10.
}
\label{fig:more}
\vspace{-3pt}
\end{figure}

\begin{table}[t!]
\caption{Ablation study of FedGELA. GE and LA mean the global ETF and local adaptation.}
\vspace{-3pt}
    \centering
    \small
    \label{tb:ablation}
\scalebox{0.87}{
    \setlength{\tabcolsep}{1.5mm}{
    \begin{tabular}{cc|cccccccccccccc}
        \toprule[2pt]
         GE & LA 
 & \multicolumn{4}{c}{SVHN}& \multicolumn{4}{c}{CIFAR10} & \multicolumn{4}{c}{ CIFAR100}& \multicolumn{2}{c}{Fed-ISIC2019} \\
\hline \multicolumn{2}{c|}{\#Partition}& \multicolumn{2}{c}{ Full Parti.} & \multicolumn{2}{c}{Partial Parti.}  & \multicolumn{2}{c}{ Full Parti.} & \multicolumn{2}{c}{Partial Parti.} & \multicolumn{2}{c}{Full Parti.} & \multicolumn{2}{c}{ Partial Parti.} &\multicolumn{2}{c}{Real World} \\
\hline \multicolumn{2}{c|}{\#Metric}& PA & GA & PA & GA& PA & GA & PA & GA & PA & GA & PA & GA& PA & GA\\\midrule
        - & - & 95.02 & 86.36& 93.15 & 88.43& 82.50 & 64.88 & 72.52 & 59.19& 69.09 & 62.80& 56.46 & 54.28& 77.27 & 73.59  \\ \midrule
        \checkmark & - & 95.92 & 88.93& 93.97 & 92.42& 83.63 & 69.70 & 77.66 & 65.56& 71.46 & 66.02& 62.67 & 58.98& 69.88 & 75.54\\ 
        - &\checkmark & 95.93& 74.84& 93.15 & 89.58& 83.97 & 63.75 & 77.76 & 61.55& 71.93 & 60.76& 58.92 & 51.95& 54.65 & 62.43 \\\midrule
        \rowcolor{lightgray!60}\checkmark & \checkmark & 96.54 & 89.07& 95.69 & 92.15& 84.61 & 69.46 & 79.95 & 65.21 & 74.23 & 66.05& 66.33 & 58.81& 79.27 & 75.85\\
        \bottomrule[2pt]
    \end{tabular}
    }}
    \label{tab:ablation}
    \vspace{-10pt}
\end{table}

\begin{table}[t]
\caption{Performance of FedGELA compared with FedAvg and the best baseline under pure PCDD settings on CIFAR10 and SVHN datasets.$P\varrho C\varsigma$ means that the dataset is divided into $\varrho$ clients and each client has $\varsigma$ classes. We show the improvement in red on each baseline compared to FedGELA.}\label{tab:purepcdd}
\small
\centering
\renewcommand\arraystretch{0.1}
\setlength{\tabcolsep}{17pt}
\scalebox{1.0}{
\begin{tabular}{c | c | c c c}
\toprule[2pt]  Dataset~(split) & Metric & FedAvg & Best Baseline & FedGELA \\ \midrule
\multirow{2}{*}{CIFAR10(P10C2)} &  PA & 92.08{\color{red}+3.76}	&94.07{\color{red}+1.77}	&95.84\\
\cmidrule{2-5} & GA  &47.26{\color{red}+12.34}&	52.02{\color{red}+7.58}&	59.60 \\ 
\midrule 
\multirow{2}{*}{CIFAR10(P50C2)} & PA&91.74{\color{red}+3.68}&	93.22{\color{red}+2.20}	&95.42 \\
\cmidrule{2-5} & GA&36.22{\color{red}+18.56}&	44.74{\color{red}+10.04}	&54.78 \\
\midrule 
\multirow{2}{*}{SVHN(P10C2)} &  PA & 95.64{\color{red}+3.11}&	97.02{\color{red}+1.73}&	98.75\\
\cmidrule{2-5} & GA  & 69.34{\color{red}+14.22}&	76.06{\color{red}+7.50}&	83.56\\ 
\midrule 
\multirow{2}{*}{SVHN(P50C2)} & PA&94.87{\color{red}+3.50}	&96.88{\color{red}+1.49}	&98.37 \\
\cmidrule{2-5} & GA& 66.94{\color{red}+10.24}	&72.97{\color{red}+4.21}	&77.18 \\
\bottomrule[2pt]
\end{tabular}
}
\end{table}

\textbf{Performance in real-world applications. }
Except for the above three benchmarks, we also verify FedGELA with other methods under a real PCDD federated application: Fed-ISIC2019. As shown in Table~\ref{tab:real}, our method achieves the best improvement of $2\%$ and $2.26\%$ relative to FedAvg and of $0.36\%$ and $1.25\%$ relative to the best baseline on personal and generic tasks respectively, which demonstrates that our method is robust to practical situations in the both views. In the Appendix~\ref{sec:morereal}, we provide more results on additional two real-world applications named FEMNIST and SHAKESPEARE to further show the effectiveness of our method in the real-world scenarios.

\subsection{Further Analysis}\label{sec:further}
\vspace{-5pt}
\begin{table}[t]
\setlength{\tabcolsep}{15pt}
\caption{{Performance of choosing different $\PHI$. Assuming the row vector of distribution matrix$(\Phi_k)_c^T$ is related to class distribution $\frac{n_{k,c}}{n_k}$ and the relationship as $Q_k(\frac{n_{k,c}}{n_k})$. Except for $Q_k(x)=x$, we have also considered employing alternative methods like employing an exponential $Q_k(x)=e^x$ or power function $Q_k(x)=x^{\frac{1}{2}}$ of the number of samples.}}\label{tab:otherphi}
\small
\centering
\renewcommand\arraystretch{0.1}
\scalebox{1.0}{
\begin{tabular}{c | c | c c c}
\toprule[2pt]  Dataset~(split) & Metric & $Q_k(x)=e^{x}$	 & $Q_k(x)=x^{\frac{1}{2}}$ & $Q_k(x)=x$(ours) \\ \midrule
\multirow{2}{*}{SVHN(IID)} &  PA & 95.12	&95.43&	94.84\\
\cmidrule{2-5} & GA  &94.32&	93.99	&94.66\\ 
\midrule 
\multirow{2}{*}{SVHN($\beta=0.5$)} & PA&96.18	&95.56	&96.27\\
\cmidrule{2-5} & GA&93.28	&93.22	&93.66\\
\midrule 
\multirow{2}{*}{SVHN($\beta=0.1$)} &  PA & 98.33&	98.21&	98.52\\
\cmidrule{2-5} & GA  & 78.95	&77.18	&78.88\\ 
\bottomrule[2pt]
\end{tabular}
}
\end{table}

\textbf{More angle visualizations.}
In Figure~\ref{fig:more}, we show the effectiveness of local adaptation
in FedGELA and verify the convergence of fixed classifier as ETF and local adaptation compared with FedAvg.

Together with Figure~\ref{fig:method}, it can be seen that, compared with FedAvg, both FedGE and FedGELA converge faster to a larger angle between all class means in global.  In the meanwhile, the angle between existing classes of FedGELA in the local is much larger, which proves FedGELA converges better than FedAvg and the adaptation brings little limits to convergence but many benefits to local performance under different levels of PCDD.

\textbf{Hyper-parameter.}
FedGELA introduces constrain~$E_H$ on the features and the length~$E_W$ of classifier vectors. We perform $L_2$ norm on all features in FedGELA, which means $E_H=1$. For the length of the classifier, we tune it as hyper-parameter. As shown in Figure~\ref{fig:ablation}, from a large range from 10e3 to 10e6 of $E_W$, our method achieves bilateral improvement compared to FedAvg on all datasets. 

\textbf{Ablation studies.}
Since our method includes two parts: global ETF and local adaptation, we illustrate the average accuracy of FedGELA on all splits of SVHN, CIFAR10/100, and Fed-ISIC2019 without the global ETF or the local adaptation or both. As shown in Table~\ref{tab:ablation}, only adjusting the local classifier does not gain much in personal or global tasks, and compared with FedGE, FedGELA achieves similar generic performance on the four datasets but much better performance on the personal tasks. 

\textbf{Performance under pure PCDD setting.}
To verify our method under pure PCDD, we decouple the PCDD setting and the ordinary heterogeneity (Non-PCDD). In Table~\ref{tab:purepcdd}, we use PxCy to denote the dataset is divided in to x clients with y classes, and in each round, 10 clients are selected into federated training. The training round is 100. According to the results, FedGELA achieves significant improvement especially $18.56\%$ to FedAvg and $10.04\%$ to the best baseline on CIFAR10 (P50C2).

\textbf{Other types of $\PHI$.} Considering the aggregation of local classifiers should align with the global classifier, which ensures the validity of theoretical analyses from both global and local perspectives, $\sum_{k=1}^N p_k\boldsymbol{\Phi}_k$ should be $\mathbf{1}$ ($\mathbf{1}$ is all-one matrix). Assuming the row vector of distribution matrix$(\Phi_k)_c^T$ is related to class distribution $\frac{n_{k,c}}{n_k}$ and the relationship as $Q_k(\frac{n_{k,c}}{n_k})$. The equation can be rewrite as:
$\gamma\sum_{k=1}^N p_kQ_k(\frac{n_{k,c}}{n_k})=\mathbf{1}$,
where $\gamma$ is the scaling constant. In our FedGELA, to avoid sharing statistics for privacy, we only find one potential way that $Q_k(\frac{n_{k,c}}{n_k})=\frac{n_{k,c}}{n_k}$  and $\gamma=\frac{1}{C}$. In this part, we have also considered employing alternative methods like employing an exponential or power function of the number of samples. As shown in the Table~\ref{tab:otherphi}, other methods need to share  $Q_k(\frac{n_{k,c}}{n_k})$  but achieve the similar performance compared to FedGELA, which exhibits the merit of our choice.

In Appendix~\ref{ap:D}, we provide more experiments from other perspectives like communication efficiency and the local burden of storing and computation, to show promise of FedGELA.

\section{Conclusion}
\vspace{-5pt}
In this work, we study the problem of \emph{partially class-disjoint data} (PCDD) in federated learning on both personalized federated learning~(P-FL) and generic federated learning~(G-FL), which is practical and challenging due to the angle collapse of classifier vectors for the global task and the waste of space for the personal task. We propose a novel method, FedGELA, to address the dilemma via a bilateral curation. Theoretically, we show the local and global convergence guarantee of FedGELA and verify the justification on the angle of global classifier vectors and on the angle between locally existing classes. Empirically, extensive experiments show that FedGELA achieves promising improvements on FedAvg under PCDD and outperforms state-of-the-art methods in both P-FL and G-FL. 
\newpage
\section*{Acknowledgement}
The work is supported by the National Key R$\&$D Program of China (No. 2022ZD0160702),  STCSM (No. 22511106101, No. 22511105700, No. 21DZ1100100), 111 plan (No. BP0719010) and National Natural Science Foundation of China (No. 62306178). Ziqing Fan and Ruipeng Zhang were partially supported by Wu Wen Jun Honorary Doctoral Scholarship, AI Institute, Shanghai Jiao Tong University. Bo Han was supported by the NSFC Young Scientists Fund No. 62006202, NSFC General Program No. 62376235, Guangdong Basic and Applied Basic Research Foundation No. 2022A1515011652, and CCF-Baidu Open Fund.


\clearpage
\newpage
\appendix
\begin{wrapfigure}{0}[0cm]{0pt}\label{fig:etf}
\includegraphics[width=0.25\textwidth]{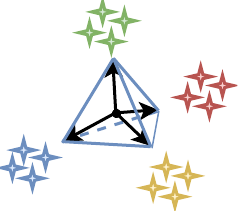}
\caption{ETF structure. Stars with different colors denote features of different classes and black arrows denote the classifier vector for each class. }
\vspace{-20pt}
\end{wrapfigure}
\section{Neural Collapse and simplex ETF}\label{ap:A}
Neural collapse~\citep{etf1} is an intuitive observation that happens at the terminal phase of a well-trained model on a balanced dataset that last-layer features converge to within-class mean, and all within-class means and their corresponding classifier vectors converge to ETF as shown in Figure~\ref{fig:etf}. The main results can be concluded as follows:
\begin{itemize}[leftmargin=15pt]
	\item (NC1) Variability of the last-layer features $\Sigma:=\operatorname{Avg}_{i, c}\{(h_{c}^i-h_c)(h_{c}^i-h_c)^T\}$ collapse within-class: $\Sigma \rightarrow \mathbf{0}$, where $h_{c}^i$ is the last-layer feature of the $i$-th sample in the $c$-th class, and $h_c$ is the within-class mean of c-th class's features.
	\item (NC2) Convergence to a simplex ETF. Last-layer features converge to within-class mean, and all within-class means and their corresponding classifier vectors converge to a simplex ETF.
	\item (NC3) Self duality: $\tilde{h}_c=\W_c /\left\|\W_c\right\|$, where  $\tilde{h}_c=(h_c-\overline{h}) /\left\|h_c-\overline{h}\right\|$ and $\W_c$ is the classifier vector of the $c$-th class.
    \item (NC4) Simplification to the nearest class center prediction: $\operatorname{argmax}_c\left\langle h, \W_c\right\rangle=\operatorname{argmin}_c\left\|h-h_c\right\|$, where $h$ is the last-layer feature.
\end{itemize}

\begin{lemma}[ETF]
    \label{lema:etf}
When solving objective defined in Eq~\eqref{eq:lemma1} in balanced C-class classification tasks with LPM and CE loss, neural collapse merges, which means $\forall 1 \leqslant i \leqslant n_c, 1 \leqslant c \leqslant C$, last layer features $H_i^*$ and corresponding classifier $\W_j^*$ converge as:
$$
 \frac{h_c^{i, *}}{\sqrt{E_H}}=\frac{\W_c^*}{\sqrt{E_W}}=m_c^*,
$$
where $m_c^*$ forms a simplex equiangular tight frame~(ETF) defined as:
$$
\M=\sqrt{\frac{C}{C-1}} U\left(\mathbf{I}_C-\frac{1}{C} \mathbf{1}_C \mathbf{1}_C^T\right),
$$
where $\M=\left[m^*_1, \cdots, m^*_C\right] \in \mathbb{R}^{d \times C}, U \in \mathbb{R}^{d \times C}$ allows a rotation and satisfies $U^T U=\mathbf{I}_C$ and $\mathbf{1}_C$ is an all-ones vector.
\end{lemma}

To analyze this phenomenon, some studies simplify deep neural networks as last-layer features and classifier~(layer-peeled model)\citep{lpm1,lpm2,lpm3,minority} with proper constraints or regularizations. In the view of layer-peeled model~(LPM), training $\W$ with constraints on the weights can be seen as training the C-class classification head $\W^{L}=\{\W_1,...,\W_C\}$ and features $H=\{h^1,...,h^N\}$ of all $n$ samples output by last layer of backbone with constraints $E_W$ and $E_H$ respectively. Therefore, $\forall 1 \leqslant c \leqslant C,  1 \leqslant i \leqslant N$, the training objective with commonly used cross-entropy loss can be described as:
\begin{equation}\label{eq:lemma1}
\begin{aligned}
\min_{H,\W^L} &\frac{1}{n} \sum_{i=1}^{N} \mathcal{L}_{CE}\left(h^i, \W\right), \\
\text { s.t. } & \left\|\W_c^L\right\|^2 \leqslant E_W, \ \left\|h^i \right\|^2 \leqslant E_H. 
\end{aligned}
\end{equation}

In the balanced dataset, as described in Lemma~\ref{lema:etf}, any solutions to this model merge neural collapse and form a simplex equiangular tight frame~(ETF), which means ETF is optimal classifier in the balanced case of LPM. 

\begin{lemma}[Fixing classifier as ETF]
    \label{lema:fix}
No matter dataset is balanced or imbalanced, fixing the classification head as ETF with scaling length of $\sqrt{E_W}$ in the layer-peeled model and optimizing the following objective:
$$
\begin{aligned}
\min_{H} &\frac{1}{n} \sum_{i=1}^{n} \mathcal{L}_{CE}\left(h^i, \sqrt{E_W}\M\right), \\
\text { s.t. } & \left\|h^i\right\|^2 \leqslant E_H, \forall 1 \leqslant i \leqslant n.
\end{aligned}
$$
Then the same solution in the Lemma~\ref{lema:etf} is obtained.
\end{lemma}
As shown in Lemma~\ref{lema:fix}, recent studies prove that no matter dataset is balanced or not, by fixing the classifier as a randomly ETF with scaling $\sqrt{E_W}$ and constraining on last layer features, LPM can reach the optimal structure as described in Lemma~\ref{lema:etf}. We also prove in Theorem~\ref{thm:global} that by fixing the classifier of all clients as ETF, in the strongly convex case, the global model can also reach the condition as Lemma~\ref{lema:etf} which meets the requirement of G-FL. 

\begin{table}[!t]
\setlength{\tabcolsep}{10pt}
\renewcommand\arraystretch{1.5}
\caption{\rm{Notations and their corresponding meanings.}}
    \centering
    \begin{tabular}{c|c}
        \toprule[2pt]
Notation &Meaning\\ 
\midrule

$F$& global loss function\\
$F_k$&local loss function of client k\\
$\tau$&local iteration in a curtain round \\
$b_{k}^{\tau}$&mini-bath of a certain iteration\\
$H$&  last layer features \\
$H^i$& last layer feature of i-th sample\\
$H_k$& last layer features of k-th client\\
$H_{k,c}^i$&last layer feature of i-th sample of c-th class in k-th client\\
$\M$& ETF matrix\\
$m_j$& j-th classifier vector in ETF\\
$\phi$& set of adjusted matrices\\
$\phi_k$& adjusted matrix of client k\\
$t$& number of communication rounds \\
$\phi_{k,c}$& adjusted weight of c-th class in client k\\
$p_k$& sample fraction of client k\\
$N$& number of clients\\
$\W$&total model\\
$\W_k$& total model of client k \\
$\W_{k}^{tE}$& model of client k after E-1 aggregation with $\tau$ additional iteration\\
$\W_{k}^{tE+\frac{1}{2}}$& model after aggregation of client k in round t \\

$\W_{k}^{\tau E+\frac{1}{2}\phi_k}$& model after aggregation and adjusted of client k in round t  \\
$\W_g$& global model\\
$\W^L$& classifier \\
$\W^{-L}$& backbone \\
$\W_k^L$& classifier of k-th client \\
$\W^{-L}_k$& backbone of k-th client \\
$\W_{k,c}^L$& classifier of c-th class in k-th client \\
$\W^{-L}_{k,c}$& backbone of c-th class in k-th client \\
$g$ & gradient\\
$n$& number of samples \\
$n_k$& sample number of client k \\
$n_{k,c}$& sample number of c-th class in client k\\
\bottomrule[2pt]
    \end{tabular}
    \label{tab:notation}
\end{table}
\section{Implementation of Theoretical Analysis}\label{ap:B}
\subsection{Notations and Assumptions}
Before starting our proof, we pre-define some notations and assumptions used in the following lemmas and theorems.
First, we make the assumptions on loss functions~$F_1, F_2,\cdots, F_N$ of all clients and their gradient functions~$\nabla F_1, \nabla F_2,\cdots, \nabla F_N$. We use $tE+\tau$ to denote $\tau$-th local iteration in round $t$, $tE$ to denote the state that just finishing local training, $tE+\frac{1}{2}$ to denote the stage after aggregation and $tE+\frac{1}{2}\PHI$ to denote the stage after local adaptation. In Assumption~\ref{asm:smooth} and Assumption~\ref{asm:strong_cvx}, we characterize the smoothness, bound on the variance of stochastic gradients and convexity of each $F_N$. In Assumption~\ref{asm:sgd_norm}, the norm of stochastic gradients is bounded, which is commonly used in many FL algorithms together with Assumption~\ref{asm:strong_cvx} to prove the global convergence~\citep{fedskip,noniid2}. An existing study points out that there is a contradiction between them~\citep{contradiction1,contradiction2}. Therefore, we show the concrete assumption description in Assumption~\ref{asm:sgd_norm_correct} and convergence guarantee without bounded norm of stochastic gradients in Theorem~\ref{thm:global_correct} and Theorem~\ref{thm:local_correct}. In Assumption~\ref{asm:heter}, the heterogeneity is reflected in the distance between local optimum $W_k^*$ and global optimum $\W^*$ and the loss deviation before and after aggregation, which is bounded by $\Gamma_1$ and $\Gamma_2$ respectively.

\begin{assumption}[L-smooth and bounded variance of stochastic gradients] 
$F_1, \cdots, F_N$ are $L$-smooth:
$$
\forall u, \forall v, 1\leq k \leq N, F_k(u) \leq F_k(v)+(u-v)^T \nabla F_k(v)+\frac{L}{2}\|u-v\|_2^2,
$$
and their variance of stochastic gradients is bounded:

\begin{equation}
\forall t\geq0, 1\leq k \leq N, \frac{1}{2}\leq \tau \leq E, \mathbb{E} \| \nabla F_k(W_k^{tE+\tau},\xi_k^{tE+\tau}) - \nabla F_k(W_k^{tE+\tau})\|^2 \le \sigma_k^2. \xi = \{\mathbf{x}, y\} \label{eq:smooth}
\end{equation}
\label{asm:smooth}
\end{assumption}

\begin{assumption}[$\mu$-strongly convex] 
$F_1, \cdots, F_N$ are $\mu$-strongly convex:
\begin{equation}
    \forall u, \forall v, 1\leq k \leq N, F_k(u) \geq F_k(v)+(u-v)^T \nabla F_k(v)+\frac{\mu}{2}\|u-v\|_2^2
\label{eq:strong_cvx}
\end{equation}

\label{asm:strong_cvx}
\end{assumption}

\begin{assumption}[Bounded norm of stochastic gradients] 
The expected squared norm of stochastic gradients is bounded:
$$\forall t\geq0, 1\leq k \leq N, \frac{1}{2}\leq \tau \leq E, \mathbb{E} \left\| \nabla F_k(\W_k^{tE+\tau},b_k^{tE+\tau}) \right\|^2  \le G^2.
$$
\label{asm:sgd_norm}
\end{assumption}

\begin{assumption}[Bounded heterogeneity] \label{asm:heter}
The deviation between local and global optimum and the deviation between local and global loss before and after aggregation are both bounded:
$$
\forall t\geq0, 1\leq k \leq N, \|\W_k^*-\W^*\| \leq \Gamma_1 \And \|\nabla F_k^{tE}-\nabla F_k^{t E+\frac{1}{2}}\|_2 \leq \Gamma_2. 
$$
\end{assumption}

\begin{assumption} [Correct bounded norm of stochastic gradients~\citep{contradiction1,contradiction2}] \label{asm:sgd_norm_correct}
Let Assumptions \ref{asm:smooth} and \ref{asm:strong_cvx} hold. Then the expected squared norm of the stochastic gradient is bounded by:

$$\mathbb{E} ||\nabla F_k (\W_k^{tE+\tau}, \xi_k^{tE+\tau})||^2 \leq 4 L \kappa \left[ F_k(\W_k^{tE+\tau}) - F_k(\W_k^*)\right] + G_k, $$
$$\text{ where } \kappa = \frac L \mu \text{ and } G_k = 2\mathbb{E} || \nabla F_k (\W_k^*, \xi_k^{tE+\tau})||^2
$$
\label{eq:new_bounded_asm}
\end{assumption}

\begin{lemma}[Results of one step SGD~\citep{fedproto}]\label{lem:sgdlocal} 
Let Assumption 1 hold. From the beginning of communication round $t+1$ to the last local update step, the loss function of an arbitrary client can be bounded as:
$$
\mathbb{E}[\mathcal{F}_k^{(t+1) E}] \leq \mathcal{F}_k^{t E+\frac{1}{2}\phi_k}-(\eta-\frac{L \eta^2}{2}) \sum_{e=\frac{1}{2}\phi_k}^{E-1}\|\nabla \mathcal{F}_k^{t E+e}\|_2^2+\frac{L E \eta^2}{2} \sigma^2 .
$$
\end{lemma}

\begin{lemma}[Results of one step SGD~\citep{fedskip,them_noniid}]\label{lem:conv_main_global} 
Assume Assumption 1 holds. If $\eta_t \leq \frac{1}{4 L}$, we have
$$
\begin{aligned}
\mathbb{E}\|W^{t+1}-W^{\star}\|^2 \leq & (1-\eta_t \mu) \mathbb{E}\|W^{t }-W^{\star}\|^2+6 L \eta_t^2 \Gamma \\
& +\eta_\tau^2 \mathbb{E}\|\mathbf{g}_\tau-\overline{\mathbf{g}}_\tau\|^2+2 \mathbb{E} \sum_{k=1}^N p_k\|W^t-W_k^{tE}\|^2,
\end{aligned}
$$
where $\Gamma=F^*-\sum_{k=1}^N p_k F_k^{\star} \geq 0$
\end{lemma}

\begin{lemma}[Math tool from Stich~\citep{stich}]Assume there are two non-negative sequences $\left\{r_\tau\right\},\left\{s_\tau\right\}$ that satisfy the relation
$$
r_{\tau+1} \leq\left(1-\alpha \gamma_\tau\right) r_\tau-b \gamma_\tau s_\tau+c \gamma_\tau^2
$$
for all $\tau \geq 0$ and for parameters $b>0, a>0, c>0$ and non-negative step sizes $\left\{\gamma_\tau\right\}$ with $\gamma_\tau \leq \frac{1}{d}$ for a parameter $d \geq$ $a, d>0$. Then, there exists weights $\omega_\tau \geq 0, W_T:=\sum_{\tau=0}^T \omega_\tau$, such that:
$$
\frac{b}{W_T} \sum_{\tau=0}^T s_\tau \omega_\tau+a r_{T+1} \leq 32 d r_0 \exp \left[-\frac{a T}{2 d}\right]+\frac{36 c}{a T}
$$
\end{lemma}

\begin{lemma}[Bounding the variance~\citep{fedskip,them_noniid}]\label{lem:conv_variance}
Assume Assumption~\ref{asm:smooth} holds. It follows that
$$
\mathbb{E} \left [\left\|\mathbf{g}_{\tau}-\overline{\mathbf{g}}_{\tau}\right\|^{2}\right]  \leq \sum_{k=1}^N p_k^2 \sigma_k^2.
$$
\end{lemma}

\begin{lemma} \label{lemma:global_phi}
(Bounding the divergence of \{$W_k^{tE}$\}~\citep{them_noniid}.). Assume Assumption~\ref{asm:sgd_norm}, that $\eta_t$ is non-increasing and $\eta_t \leq 2\eta_{t+E}$ for all $t \geq 0$. It follows that:
$$
\mathbb{E} \left [   \sum_{k=1}^N p_k \|\W^t-\W_k^{tE}\|^2  \right] \leq 4\eta_t^2(E-1)^2 G^2 + \sum_{k=1}^N p_k ||\PHI_k\W^L - \W^L||
$$

\end{lemma}

\begin{proof}
    Different from the Lemma in \citep{assumption3}, we consider the ETF structure in $W$. Therefore, for any $t > 0$ and $k=1,2,\cdots, N$, we use the fact that $\eta_t$ is non-increasing and $\eta_{tE} \leq 2 \eta_t$, then

\begin{align}
    \mathbb{E} \sum_{k=1}^n p_k \|\W^t-\W_k^{tE}\|^2  = & \mathbb{E} \sum_{k=1}^N p_k (\|\W^t-\W_k^{tE}\|^2 + ||\PHI_k\W^L - \W^L||^2) \\
     \mathop{\leq} \limits_{\text{SGD}} &\sum_{k=1}^N p_k \left( \mathbb{E} \sum_{\tau=2}^{E} (E-1) \| \eta_{\tau} \nabla F_k(\W_k^{t\tau}, \xi_k^{t\tau}) \|^2 + ||\PHI_k\W^L - \W^L||^2 \right ) \\
     \mathop{\leq}\limits_{\text{Assumption~\ref{asm:sgd_norm}}} & \sum_{k=1}^N p_k \left(\eta_{tE}^2 (E-1)^2 G^2 + ||\PHI_k\W^L - \W^L||^2 \right ) \\
     \mathop{\leq} \limits_{\eta_{tE} \leq 2 \eta_t}  & 4\eta_t^2(E-1)^2 G^2 + \sum_{k=1}^N p_k ||\PHI_k\W^L - \W^L||.
\end{align}
\end{proof}

\subsection{Proof of Theorem 1}

\begin{proof}
Similar to~\citep{them_noniid}, from Lemma~\ref{lem:conv_main_global}, Lemma~\ref{lemma:global_phi} and Lemma~\ref{lem:conv_variance}, let $\gamma=\max\{8\kappa, E\}$ and $\eta_{tE} \leq 2\eta_t$, it follows that
$$
\mathbb{E}[F(\W_g)] - F(\W^*) \leq \frac{\kappa}{\gamma+T-1} \left (  \frac{2B}{\mu} + \frac{\mu \gamma} {2} \mathbb{E} || \W^1 - \W^* ||^2 \right),
$$
which uses the same proof technique in \citep{them_noniid}. And we apply the cross-entropy loss for $F$, then $F(\W) = -\log[(\W^L)^T h_c^i]$ for class $c$ on sample $i$. Then we have
$$
\mathbb{E}[F(\W_g)] - F(\W^*) = \mathbb{E} \left[\log \frac{(\W^{L,*})^T h^{i,*}_{c}}{(\W_g^L)^T h^{i}_{c}} \right ].
$$
So Theorem~\ref{thm:global} is proved.

\end{proof}

\subsection{Proof of Theorem 2}
\begin{proof}

We start our proof from one step of SGD defined in Lemma~\ref{lem:sgdlocal}:
$$
\mathbb{E}[F_k^{(t+1) E}] \leq F_k^{t E+\frac{1}{2}\phi_k}-(\eta-\frac{L \eta^2}{2}) \sum_{e=\frac{1}{2}\phi_k}^{E-1}\|\nabla F_k^{t E+e}\|_2^2+\frac{L E \eta^2}{2} \sigma^2 .
$$
We can take apart the $F_k^{t E+\frac{1}{2} \PHI_k}$ and have the fact that:
$$
\begin{aligned}
\| F_k^{t E+\frac{1}{2} \PHI_k} \|& =\left\|F_k^{t E}+F_k^{t E+\frac{1}{2} \PHI_k}-F_k^{t E}\right\| \\
& \leqslant\left\|F_k^{t E}\right\|+\left\|F_k^{t E+\frac{1}{2} \PHI_k}-F_k^{t E}\right\| \\
& \leqslant \|F_k^{t E} \|+\| F_k^{t E+\frac{1}{2} \PHI_k}-F_k^{t E+\frac{1}{2}}+F_k^{t E+\frac{1}{2}}-F_k^{t E} \| \\
& \leqslant\left\|F_k^{t E}\right\|+\left\|F_k^{t E+\frac{1}{2} \PHI_k}-F_k^{t E+\frac{1}{2}}\right\|+\| F_k^{t E+\frac{1}{2}}-F_k^{t E} \| \\
& \leqslant\left\|F_k^{tE} \|+L\right\| \W_k^{t E+\frac{1}{2} \PHI_k}-\W_k^{t E+\frac{1}{2}} \|+\Gamma_1 \\
& =\left\|F_k^{t E}\right\|+L\left\|\PHI_k\W^L-\W^L\right\| +\Gamma_1
\end{aligned}
$$
Take it back to the original equation, therefore we have the:
$$
\mathbb{E}[F_k^{(t+1) E}] \leq\left\|F_k^{t E}\right\|-(\eta-\frac{L \eta^2}{2}) \sum_{e=\frac{1}{2}\phi_k}^{E-1}\|\nabla F_k^{t E+e}\|_2^2+\frac{L E \eta^2}{2} \sigma^2 +L\left\|\PHI_k\W^L-\W^L\right\|+\Gamma,
$$
By applying Assumption~\ref{asm:sgd_norm} that:$
\mathbb{E}\left\|\nabla F_k\left(u, b_k^\tau\right)\right\|^2 \leq G^2$, results will be:
$$
\mathbb{E}[F_k^{(t+1) E}] \leqslant F_k^{t E}-(\eta-\frac{L}{2} \eta^2) E G^2+\frac{L E \eta^2}{2} \sigma^2 +L\left\|\PHI_k\W^L-\W^L\right\|+\Gamma.
$$
Here we complete our proof.
\end{proof}

\subsection{Contradictory of the Assumptions and Correction}

\subsubsection{Contradictory on Assumption~\ref{asm:sgd_norm}.}

We will prove that if Assumptions \ref{asm:smooth} and \ref{asm:strong_cvx} hold, the stochastic gradients cannot be uniformly bounded. 

\begin{proof}
If Assumptions \ref{asm:smooth} and \ref{asm:strong_cvx} hold, which means $F_k$ is both \textit{L-smooth} and \textit{$\mu$-strong convex}, we have:
\begin{equation}
    2\mu[F_k(\W)-F_k(\W^*)] \leq \left | |\nabla F_k (\W)| \right |^2
    \label{eq:lower_bound_sgd}
\end{equation}
The proof of  \eqref{eq:lower_bound_sgd} will be given below. And under the false stochastic gradients uniformly bounded assumption \ref{asm:sgd_norm}, we have $\mathbb{E}[||\nabla F_k(\W_k^{tE+\tau}, \xi_k^{tE+\tau})||^2] \leq G^2$. So we get 
\begin{equation}\label{eq:proof_a4_upper}
    \begin{aligned}
        2\mu[F_k(\W)-F_k(\W^*)] &\leq \left | |\nabla F_k (\W)| \right |^2 \\&\leq ||\mathbb{E}[\nabla F_k(\W_k^{tE+\tau}, \xi_k^{tE+\tau})] ||^2 \\&\leq \mathbb{E}[||\nabla F_k(\W_k^{tE+\tau}, \xi_k^{tE+\tau})||^2] \\&\leq G^2
    \end{aligned}
\end{equation}
Therefore, we have the result that $F_k(\W) - F_k(\W^*) \leq \frac{G^2}{2\mu}$. Using the strong convex in \eqref{eq:strong_cvx} with $\W = \W^*$ that $\nabla F_k(\W^*) = 0$, we will have:
\begin{equation}\label{eq:proof_a4_low}
    F_k(\mathbf{v}) - F_k(\W^*) \geq (\mathbf{v} - \W^*)^T  \nabla F_k(\W^*) + \frac{\mu}{2} || \mathbf{v} - \W^* ||^2 = \frac{\mu}{2} || \mathbf{v} - \W^* ||^2
\end{equation}
By combining  \eqref{eq:proof_a4_low} and  \eqref{eq:proof_a4_upper}, it follows that
$$
\frac{G^2}{2\mu} \geq F_k(\W) - F_k(\W^*) \geq \frac{\mu}{2} || \W - \W^* ||^2,
$$
\begin{equation}\label{eq:proof_a4_final}
    || \W - \W^* ||^2 \leq \frac{G^2}{\mu^2}.
\end{equation}
where  \eqref{eq:proof_a4_final} is clearly wrong for sufficiently large $||\W - \W^*||^2$.

\end{proof}

\subsubsection{Proof of corrected Assumption~\ref{asm:sgd_norm_correct}.}

\begin{proof}
Note that:
\begin{align}
    ||a||^2 = & ||a-b+b||^2 \leq 2||a-b||^2 + 2||b||^2 \\
    &\implies \frac{1}{2}||a||^2 - ||b||^2 \leq ||a-b||^2 \label{eq:mean_inequality}
\end{align}
If Assumptions \ref{asm:smooth} and \ref{asm:strong_cvx} hold, combined with  \eqref{eq:mean_inequality} we have:
\begin{equation*}
\begin{aligned}
    \frac{1}{2}\mathbb{E}[||\nabla F_k(\W_k^{tE+\tau},\xi_k^{tE+\tau})||^2] - & \mathbb{E}[||\nabla F_k(\W^{*}_k, \xi_k^{tE+\tau})||^2] \\
    = \ & \mathbb{E}\left[  \frac{1}{2}||\nabla F_k(\W_k^{tE+\tau},\xi_k^{tE+\tau})||^2 - ||\nabla F_k(\W^{*}_k, \xi_k^{tE+\tau})||^2\right] \\
   \leq  \ & \mathbb{E}\left[ || \nabla F_k(\W_k^{tE+\tau}, \xi_k^{tE+\tau}) - \nabla F_k(\W^{*}_k, \xi_k^{tE+\tau}) ||^2 \right] \\ 
    \mathop{\leq}\limits_{\text{Eq.\eqref{eq:smooth}}} & L^2 ||\W_k^{tE+\tau} - \W^{*}_k||^2 \\ \mathop{\leq}\limits_{\text{Eq.\eqref{eq:strong_cvx}}} & \ \frac{2L^2}{\mu}[F_k(\W_k^{tE+\tau}, \xi_k^{tE+\tau}) - F_k(\W^{*}_k,\xi_k^{tE+\tau})] \\ 
    = \ & 2L\kappa [F_k(\W_k^{tE+\tau}, \xi_k^{tE+\tau}) - F_k(\W^{*}_k,\xi_k^{tE+\tau})]
\end{aligned}
\end{equation*}
So we get: $\mathbb{E}[|| \nabla F_k(\W_k^{tE+\tau}, \xi_k^{tE+\tau}) ||^2] \leq 4 L \kappa [F_k(\W_k^{tE+\tau}) - F_k(\W^{*}_k)] + G_k$.
\end{proof}

\subsubsection{Correction.}
In this part, we provide convergence results without the bounded norm of stochastic gradient defined in Assumption~\ref{asm:sgd_norm}. In Theorem~\ref{thm:global_correct} and Theorem~\ref{thm:local_correct}, we show the corrected results of global and local convergence, respectively.

\begin{theorem}[Global Convergence]\label{thm:global_correct}
If $F_1,...,F_N$ are all L-smooth, $\mu$-strongly convex, and the variance and norm of $\nabla F_1,...,\nabla F_N$ are bounded by $\sigma$ and $G$. Choose $\kappa=L / \mu$ and $\gamma =\frac {32}{k(\mu-k)} L^2 \kappa (E-1)^2-1$, for all classes $c$ and sample $i$, expected global representation by cross-entropy loss will converge to:

$$
\mathbb{E} \left[\log \frac{(\W^{L,*})^T h^{i,*}_{c}}{(\W_g^L)^T h^{i}_{c}} \right ] \leq \frac{\kappa}{\gamma+T-1} \left (  \frac{2B}{\mu} + \frac{\mu \gamma} {2} \mathbb{E} || \W^1 - \W^* ||^2 \right),
$$

where in FedGELA, $B = \sum_{k=1}^N (p_k^2 \sigma^2 + p_k ||\PHI_k\W^L - \W^L|| ) + 6 L \Gamma_1 + 8(E-1)^2G^2$ and $G=\sum_{k=1}^{K}p_k G_k = 2 \sum_{k=1}^K p_k \mathbb{E}[|| \nabla F_k(\W^{*}_k,\xi_k^{tE+\tau}) ||^2]$. Since $\W^{L}=\W^{L,*}$ and $(\W^{L,*})^T h^{i,*}_{c_i} \geq \mathbb{E}[(\W^L)^T h^{i}_{c_i}]$, $h^{i}_{c_i}$ will converge to $h^{i,*}_{c_i}$.
\end{theorem}
Similar to Theorem~\ref{thm:global}, the variable $B$ in Theorem~\ref{thm:global_correct} represents the impact of algorithmic convergence ($p_k^2 \sigma^2$), non-iid data distribution ($6 L \Gamma_1$), and stochastic optimization ($8(E-1)^2G^2$). The only difference between FedAvg, FedGE, and our FedGELA lies in the value of $B$ while others are kept the same. FedGE and FedGELA have a smaller $G$ compared to FedAvg because they employ a fixed ETF classifier that is predefined as optimal. FedGELA introduces a minor additional overhead ($p_k || \PHI_k\W^L - \W^L||$) on the global convergence of FedGE due to the incorporation of local adaptation to ETFs.

The cost might be negligible, as $\sigma$, $G$, and $\Gamma_1$ are defined on the whole model weights while $p_k \|\PHI_k\W^L-\W^L\|$ is defined on the classifier. To verify this, we conduct experiments in Figure~\ref{fig:method}(a), and as can be seen, FedGE and FedGELA have similar quicker speeds and larger classification angles than FedAvg.

\begin{theorem}[Local Convergence]\label{thm:local_correct}
If $F_1,...,F_N$ are L-smooth and the heterogeneity is bounded by $\Gamma_2$, clients' expected local loss satisfies:
$$
F_K(\W_k^{t E+\frac{1}{2} \PHI})-F_k^*(\W_k^*) \leq L\|\W_k^{t E+\frac{1}{2}}-\W^*\|+D,
$$
where in FedGELA, $D=\Gamma_2+\|\PHI_k\W^L-\W^L\| E_w$, which means the local convergence is highly related to global convergence and bounded by D.
\end{theorem}

In Theorem~\ref{thm:local_correct}, only ``D'' is different on the convergence among FedAvg, FedGE, and FedGELA. The local convergence is highly related to global convergence and bounded by D. Adapting the local classifier will introduce additional cost of $L\left\|\PHI_k\W^L-\W^L\right\|$, which might limit the speed of local convergence. However, FedGELA might reach better local optimal by adapting the feature structure.  As introduced and verified in Figure~\ref{fig:method}~(c) in the main pape, the adapted structure expands the decision boundaries of existing major classes and better utilizes the feature space wasted by missing classes.
\begin{proof}
We can prove the theorem by inserting $\W_k^*$ and taking apart the local loss function:
$$
\begin{aligned}
& F_k(\W_k^{t E+\frac{1}{2} \phi})-F_k^*(\W_k^*) \\
& \leqslant\|F_k(\W_k^{t E+\frac{1}{2} \PHI})-F_k(\W^*)\|+\|F_k(\W^*)-F_k^*(\W_k^*)\| \\
& \leq L \|\W_k^{t E+\frac{1}{2} \PHI}-\W^*\|+\Gamma_2 \\
& =L\|\W_k^{t E+\frac{1}{2}}-\W^*\|+\Gamma_2+\|\PHI_k \W^L-\W^L\| \\
&
\end{aligned}
$$
Here we complete the proof. The last and the second last inequalities are derived from L-smooth and bounded heterogeneity defined in Assumption~\ref{asm:smooth} and Assumption~\ref{asm:heter} respectively.
\end{proof}

\subsection{Implementation of the Justification Experiments.}
To verify the contradiction of the local objective and global objective, we track the angle of classifier vectors between locally existing classes and locally missing classes in an individual client. We denote "existing angle" as the angle of classifier vectors belonging to classes that exist in a local client while "missing angle" is the angle of classifier vectors belonging to non-existing classes. In Sec.~\ref{sec:motivation} and shown in Figure~\ref{fig:motivation}, the tracking experiment is conducted on CIFAR10 with 10 clients under Dir~($\beta=0.1$). To verify the effectiveness and convergence of FedGELA, we track the angle between class means of locally existing classes and all classes in local and global, respectively. In Sec.~\ref{sec:analysis} and illustrated in Figure~\ref{fig:method}, the tracking experiment is conducted on CIFAR10 with 50 clients under Dir~($\beta=0.2$). In the experiment, we also provide more results under different situations illustrated in Figure~\ref{fig:more}. 

\section{Implementation of the Experiment}\label{ap:C}
\subsection{Model}
Resnet18 backbone~\citep{moon, fedproto,fedbabu,fedskip,fedlc,edgecloud,zhang2} is commonly used in many federated experiments on CIFAR10 and CIFAR100 datasets, here we also use it as the backbone for SVHN, CIFAR10 and CIFAR100. Since there are many algorithms that are feature-based like MOON and FedProto, therefore we use one layer of FC as the projection layer~(the hidden size is 84 for SVHN and CIFAR10 and 512 for CIFAR100) followed by classification head. For FedGELA and FedGE, the model is a backbone, projection layer with a simplex ETF or adapted ETF. For Fed-ISIC2019, we follow the setting of Flambly and use pre-trained Efficientnet b0 with the same projection layer~(the hidden size is 84) as the model. 

\begin{figure}[ht]
\centering  
\subfigure[CIFAR10~$(\beta=10000$).]{
\label{pic:c10iid}
\includegraphics[width=0.31\textwidth]{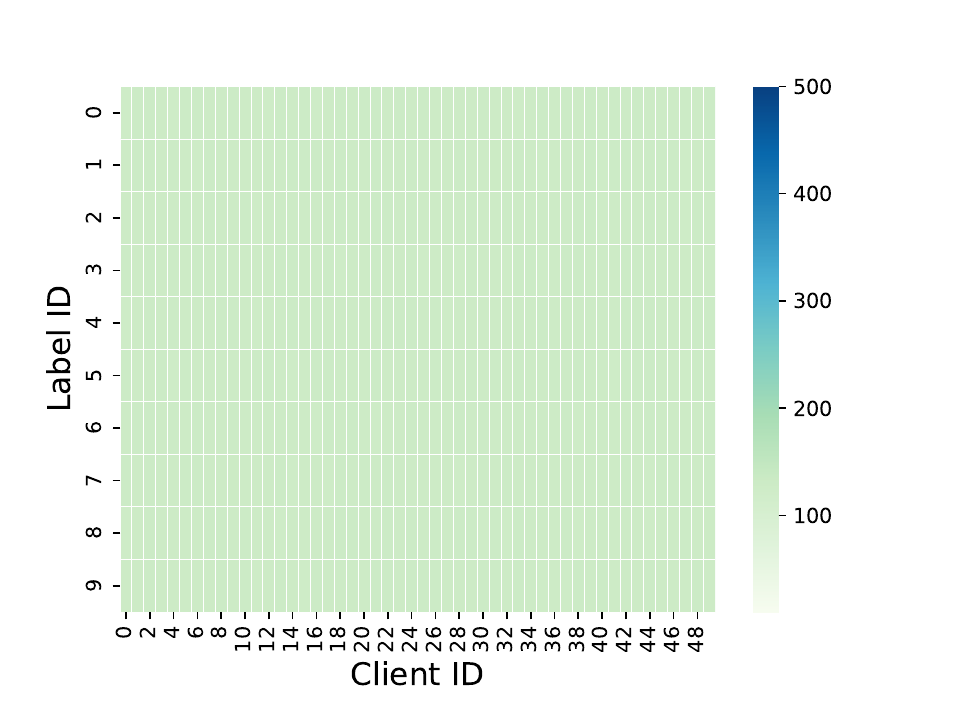}}
\subfigure[CIFAR10~$(\beta=0.5$).]{
\label{pic:c10b05}
\includegraphics[width=0.31\textwidth]{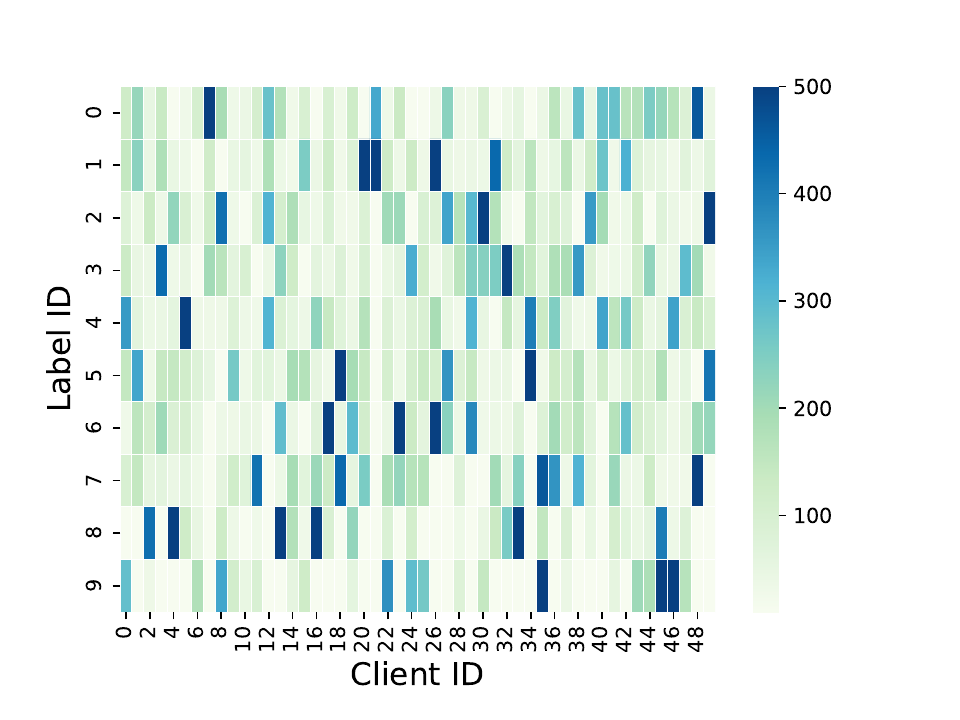}}
\subfigure[CIFAR10~$(\beta=0.2$).]{
\label{pic:c10b02}
\includegraphics[width=0.31\textwidth]{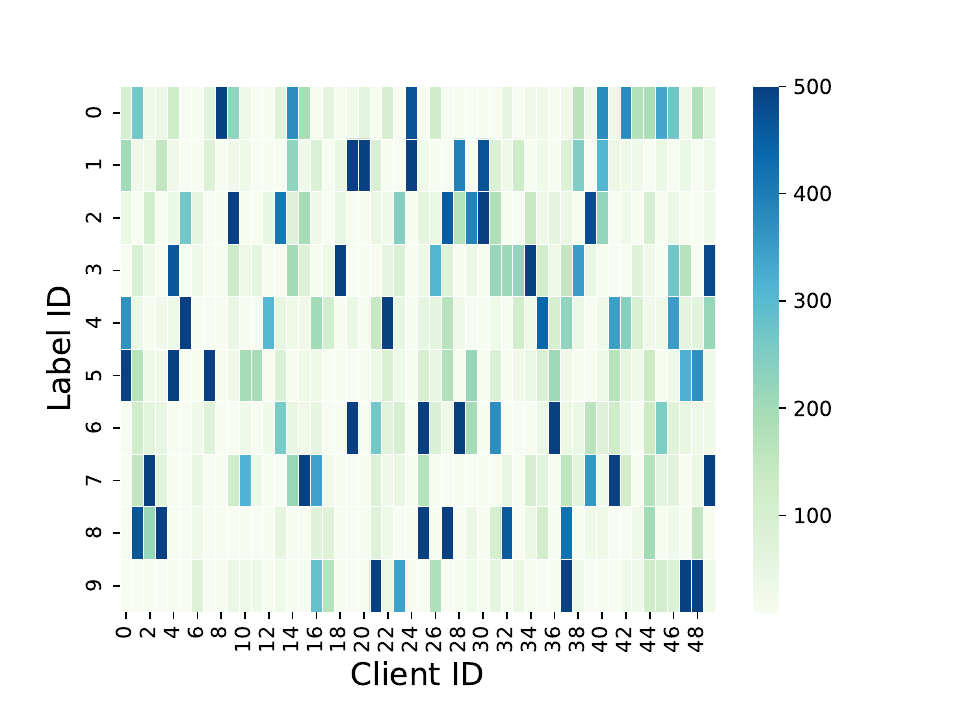}}
\subfigure[CIFAR100~$(\beta=10000$).]{
\label{pic:c100iid}
\includegraphics[width=0.31\textwidth]{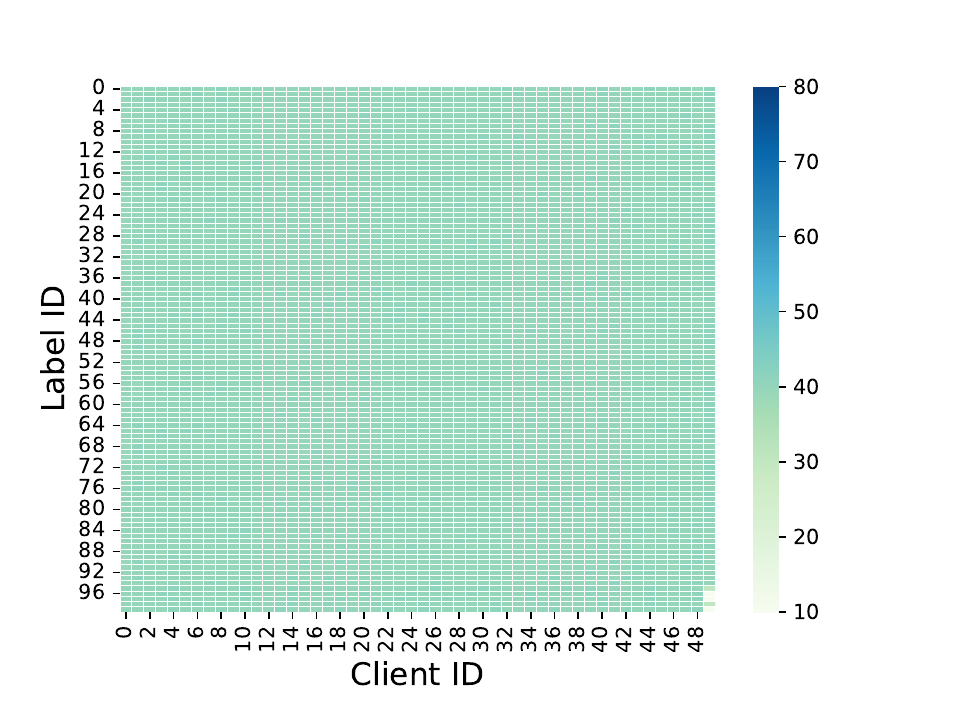}}
\subfigure[CIFAR100~$(\beta=0.5$).]{
\label{pic:c100b05}
\includegraphics[width=0.31\textwidth]{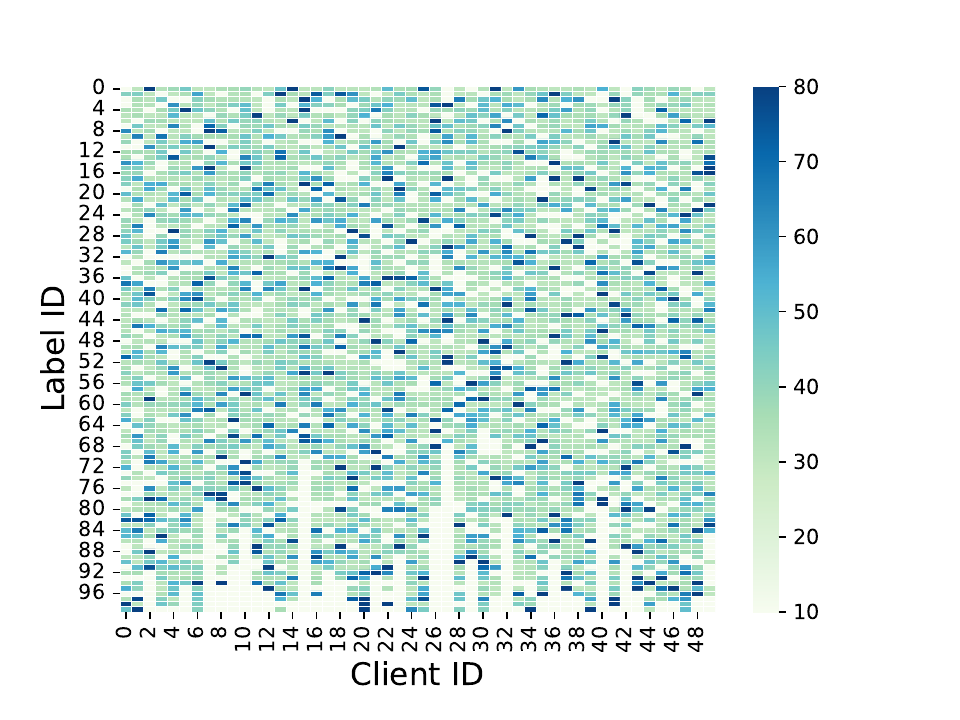}}
\subfigure[CIFAR100~$(\beta=0.2$).]{
\label{pic:c100b02}
\includegraphics[width=0.31\textwidth]{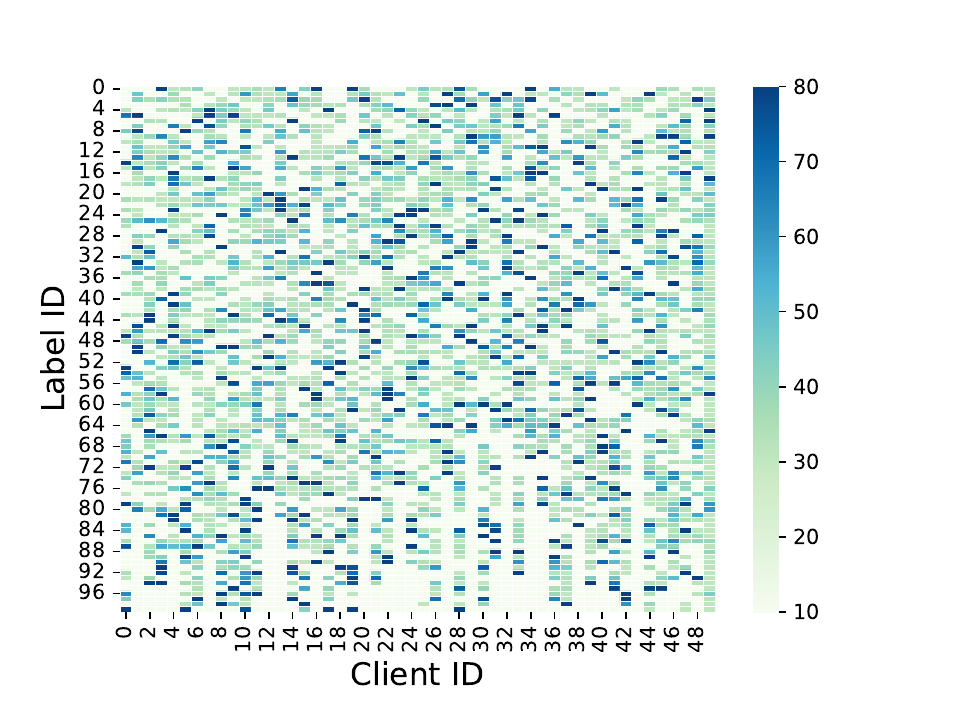}}
\caption{Heatmap of data distribution under Dirichlet distribution with different $\beta$. The empty color denotes there is no sample of a category in a client, indicating the PCDD situation.}
\label{fig:heatmap}
\end{figure}

\subsection{Partition Strategy}
Dirichlet distribution~(Dir~($\beta$)) is practical and commonly used in FL settings~\citep{moon,fedproto,fedskip,fedrs,fedrod,fedbabu,fedlc}. As in many recent works, we deploy $Dir~(\beta=10000)$ to simulate the almost IID situations and $Dir~(\beta=0.5, 0.2, 0.1)$ to simulate the different levels of Non-IID situations. As shown in Figure~\ref{fig:heatmap}, we provide the data distribution heatmap among clients of SVHN, CIFAR10, and CIFAR100 under Dirichlet distribution with different $\beta$. It can be seen that Dirichlet distribution can also generate practical PCDD data distribution. We also provide the data distribution heatmap of Fed-ISIC2019 shown in Figure~\ref{fig:isic_dataset}. In Fed-ISIC2019, there exists a true PCDD situation that needs to be solved. To verify full participating~(10,10) and straggler situations when client numbers are increasing, we split SVHN, CIFAR10, and CIFAR100 into 10 and 50 clients, and in each round, 10 clients are randomly selected into federated training. In FedISIC2019, there are 6 clients with 8 classes of samples and we split the 6 clients into 20 clients and in each round randomly select 10 clients to join the federated training.

\begin{wrapfigure}{r}{0.4\linewidth}\label{fig:isic_dataset}
\vspace{-60pt}
\begin{center}
    \includegraphics[width=0.4\textwidth]{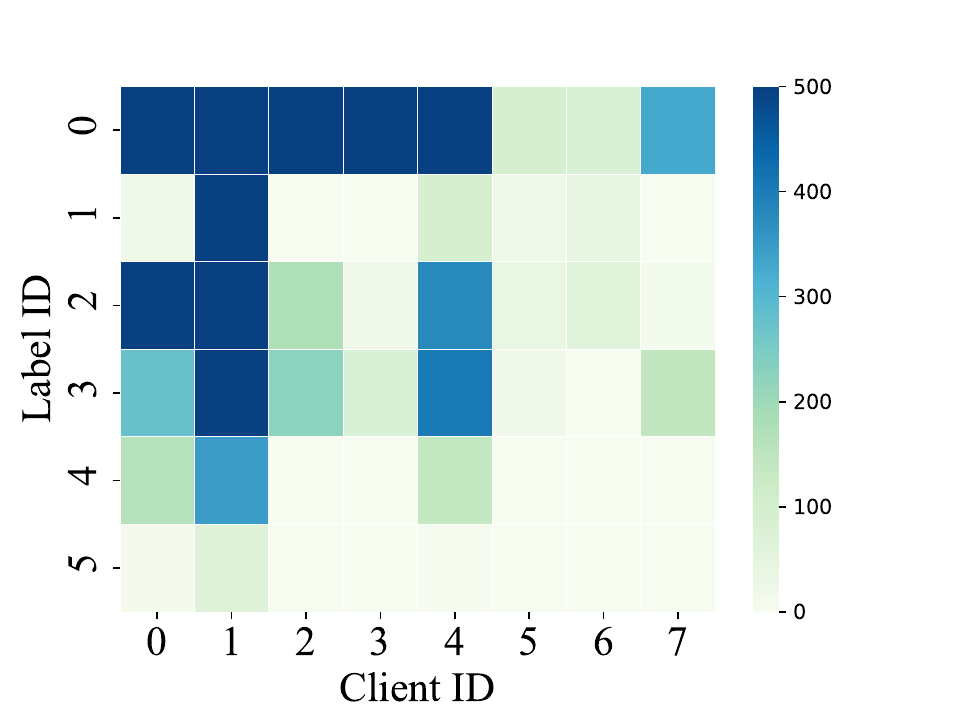}
\end{center}
\vspace{-10pt}
\caption{Data distribution of ISIC2019 dataset. The empty color denotes there is no sample of a category in a client, indicating the PCDD situation.}
\vspace{-20pt}
\end{wrapfigure}

\subsection{Training and Algorithm-Specific Params}
Since the aim of our work is not to acquire the best performance on the four datasets, we use stable and almost the best training parameters in FedAvg and applied on all other methods. We verify and use SGD as the optimizer with learning rate lr=0.01, weight decay 1e-4, and momentum 0.9. The batch size is set to 100 and the local epoch is 10. We have verified that such learning rates and local epochs are much more stable and almost the best. Note that, only training params are equal with FedAvg, but method-specific parameters like proximal terms in FedProx and contrastive loss in MOON are carefully tuned.

\begin{table}[!t]
\caption{Mean and std of averaged personal and generic performance on all settings on the four datasets of FedAvg, best baselines, and our FedGELA. we run three different seeds and calculate the mean and std for all methods.}
    \centering
\scalebox{0.85}{
    \setlength{\tabcolsep}{1mm}{
    \begin{tabular}{c|cccccccc}
        \toprule[2pt]
         Method 
 & \multicolumn{2}{c}{SVHN}& \multicolumn{2}{c}{CIFAR10} & \multicolumn{2}{c}{CIFAR100}& \multicolumn{2}{c}{Fed-ISIC2019} \\
\midrule 
\# Metric& PA & GA & PA & GA& PA & GA & PA & GA\\\midrule
         FedAvg &$94.09_{\pm 0.15}$ & $87.39_{\pm 0.20}$& $77.51_{\pm 0.29}$ & $62.04_{\pm 0.26}$& $62.78_{\pm 0.43}$ & $58.54_{\pm 0.39}$ & $77.27_{\pm 0.19}$ & $73.59_{\pm 0.17}$\\ \midrule
        Best Baseline &$95.18_{\pm 0.19}$ & $88.85_{\pm 0.21}$& $80.61_{\pm 0.33}$ & $66.07_{\pm 0.24}$& $64.28_{\pm 0.46}$ & $60.31_{\pm 0.32}$ & $78.91_{\pm 0.13}$ & $74.98_{\pm 0.21}$\\ \midrule
         FedGELA (ours) & $96.12_{\pm 0.13}$& $90.61_{\pm 0.19}$& $82.28_{\pm 0.16}$ & $67.34_{\pm 0.15}$& $70.28_{\pm 0.36}$ &$62.43_{\pm 0.28}$ & $79.29_{\pm 0.19}$ & $75.85_{\pm 0.16}$\\
        \bottomrule[2pt]
    \end{tabular}
    }
    }
    \label{tab:std}
\end{table}

\subsection{Mean and STD}
In all our experiments, we run three different seeds and calculate the mean and std for all methods. In Table~\ref{tab:real} and Figure~\ref{fig:ablation}, we report the both mean and std of results while for other experiments, due to the limited space, we only report mean results. Therefore in this part, we additionally provide the mean with std of averaged performance on all partitions of FedAvg, best baselines, and our FedGELA in Table~\ref{tab:std}.

\section{More Information of FedGELA}\label{ap:D}

\begin{wrapfigure}{r}{0.5\linewidth}
\vspace{-35pt}
\begin{center}
   \includegraphics[width=0.48\textwidth]{ 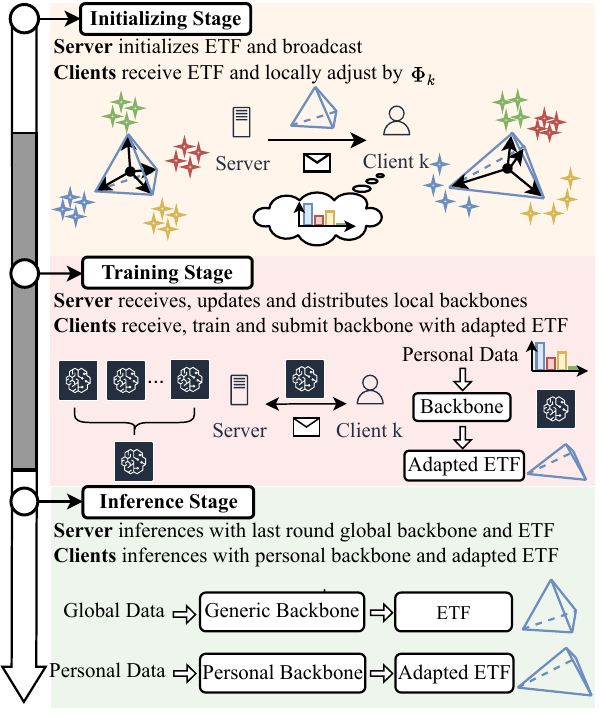} 
\end{center}
\vspace{-10pt}
\caption{Total framework of FedGELA}\label{fig:workflow}
\vspace{-20pt}
\end{wrapfigure}

\subsection{Work Flow of FedGELA}   
Except for the algorithm of our FedGELA shown in Algorithm~\ref{alg:fedgela}, we also provide the workflow of FedGELA. As shown in Figure~\ref{fig:workflow}, the FedGELA can be divided into three stages, namely the initializing stage, the training stage, and the inference stage. In the initializing stage, the server randomly generates a simplex ETF as the classifier and sends it to all clients. In the meanwhile, clients adjust it based on the local distribution as Eq~\eqref{eq:adjust}. At the training stage, local clients receive global backbones and train with adapted ETF in parallel. After $E$ epochs, all clients submit personal backbones to the server. In the server, personal backbones are received and aggregated to a generic backbone, which is distributed to all clients in the next round. At the inference stage, on the client side, we obtain a generic backbone with standard ETF to handle the global data while on the client side, a personal backbone with adapted ETF to handle the personal data.

\begin{table}[!t]
\centering
\small

\caption{\rm{Communication efficiency of FedGELA compared with FedAvg and a range of state-of-the-art methods on CIFAR10 under different settings. Communication efficiency is defined as the communication rounds that need to reach the best global accuracy of FedAvg within curtain rounds. We use '-' to denote the situation that the algorithm can not reach the best accuracy of FedAvg during limited communication rounds.}}
    \setlength{\tabcolsep}{1pt}
    \scalebox{0.9}{
    \begin{tabular}{c|ccccccccccccc}
        \toprule[2pt]
        \multirow{1}{*}{Method}& 
        \multicolumn{2}{c}{IID, Full}&\multicolumn{2}{c}{$\beta=0.5$, Full}& 
        \multicolumn{2}{c}{$\beta=0.1$, Full}& 
        \multicolumn{2}{c}{IID, Partial}& 
        \multicolumn{2}{c}{$\beta=0.5$, Partial}& 
        \multicolumn{2}{c}{$\beta=0.2$, Partial}
        \\
        \cmidrule(r){2-3}\cmidrule(r){4-5}\cmidrule(r){6-7}\cmidrule(r){8-9}\cmidrule(r){10-11}\cmidrule(r){12-13}
        \footnotesize{(CIFAR10)}&  Commu. & Speedup&  Commu. & Speedup &Commu. & Speedup&Commu. & Speedup&Commu. & Speedup&Commu. & Speedup\\
        \midrule
        FedAvg&$100$ &$1\times$&$100$ &$1\times$ &$100$  &$1\times$&$200$ &$1\times$&$200$ &$1\times$ &$200$  &$1\times$\\
        FedProx&$42$ &$2.38\times$&$86$ &$1.16\times$ &$83$  &$1.20\times$&$105$ &$1.90\times$&$139$ &$1.44\times$ &$152$  &$1.32\times$\\
        MOON&$42$ &$2.38\times$&$53$ &$1.89\times$ &$79$  &$1.27\times$&$98$ &$2.04\times$&$136$ &$1.47\times$ &$145$  &$1.38\times$\\
        FedRS&$39$ &$2.56\times$&$65$ &$1.54\times$ &$84$  &$1.19\times$&$103$ &$1.94\times$&$136$ &$1.47\times$ &$126$  &$1.59\times$\\
        FedLC&$47$ &$2.13\times$&$45$ &$2.22\times$ &$82$  &$1.22\times$&$113$ &$1.77\times$&$118$ &$1.69\times$ &$121$  &$1.65\times$\\
        FedRep&$-$ &$-$&$-$ &$-$ &$-$  &$-$&$-$ &$-$&$-$ &$-$ &$-$  &$-$\\
        FedProto&$-$ &$-$&$-$ &$-$ &$-$  &$-$&$-$ &$-$&$-$ &$-$ &$-$  &$-$\\
        FedBABU&$63$ &$1.59\times$&$60$ &$1.67\times$ &$-$  &$-$&$-$ &$-$&$-$ &$-$ &$-$  &$-$\\
        FedRod&$55$ &$1.82\times$&$51$ &$1.96\times$ &$77$  &$1.30\times$&$80$ &$2.50\times$&$112$ &$1.79\times$ &$142$  &$1.41\times$\\
        \midrule
\rowcolor{lightgray!40}\textbf{FedGELA}&$\textbf{42}$ &$\textbf{2.38}\times$&$\textbf{52}$  &$\textbf{1.92}\times$&$\textbf{67}$&$\textbf{1.49}\times$&$80$ &$2.50\times$&$114$ &$1.75\times$ &$119$  &$1.68\times$\\
        \bottomrule[2pt]
    \end{tabular}}
    \label{tab:commu}
\end{table}

\subsection{Communication Efficiency}
Communication cost is a much-watched concern in federated learning. Since our algorithm does not introduce additional communication overhead, we compare the number of communication rounds required for all algorithms to reach FedAvg's best accuracy. Since PA is hard to track and highly related to GA as shown in Theorem~\ref{thm:global_correct}, here we only compare the communication rounds that are required to reach the best GA of FedAvg. As shown in Table~\ref{tab:commu}, we provide communication rounds and speedup to FedAvg compared with a range of the state of the art methods. It can see that P-FL algorithms are hard to reach the global accuracy of FedAvg since they limit the generic ability of the local model while our FedGELA achieves almost the best communication efficiency in all settings.
\begin{figure}[!t]
\centering
\includegraphics[width=0.8\textwidth]{ 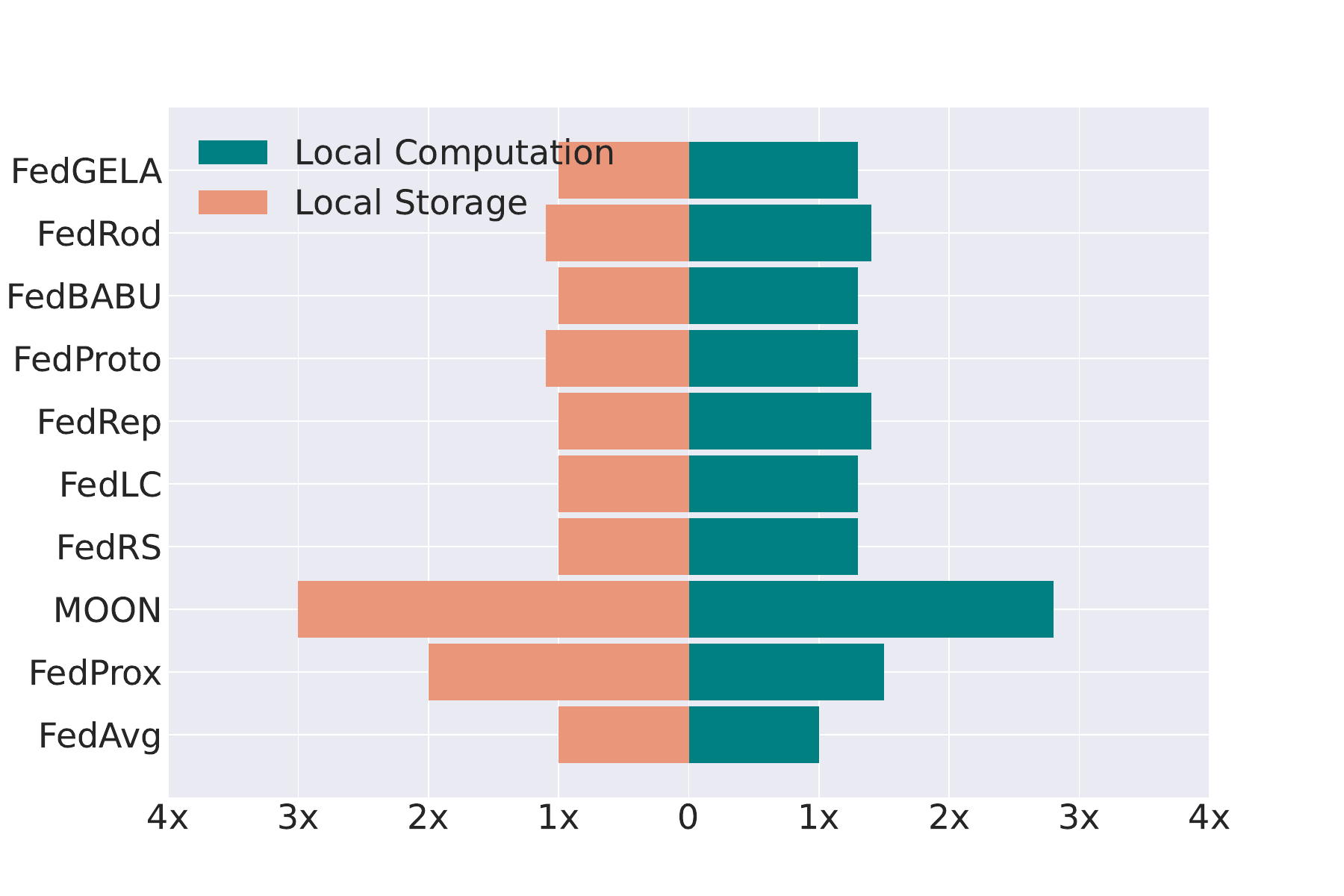}
\caption{Local computation and storage of FedGELA compared with FedAvg and a range of the state-of-the-art methods.}\label{fig:burden}
\end{figure}
\subsection{Local Burden: Storing and Computation}

In real-world federated applications, local clients might be mobile phones or other small devices. Thus, the burden of local training can be the bottleneck for clients. In Figure~\ref{fig:burden}, we compute the number of parameters that need to be saved in local clients and the average local computation time in each round. As can be seen, MOON requires triple storing memory than FedAvg, while FedGELA keeps the same level as FedAvg. In terms of local computation time, FedGELA introduces negligible computing time to local training, indicating the efficiency of our method on the local burden concerns.

\subsection{Other Backbones}

\begin{table}[!t]
\caption{Averaged performance of FedGELA compared with FedAvg and a range of state-of-the-art methods on SVHN under all settings with different backbones, namely Simple-CNN, ResNet18, and ResNet50. }
    \centering
\scalebox{1.0}{
    \setlength{\tabcolsep}{4.5mm}{
    \begin{tabular}{c|cccccccc}
        \toprule[2pt]
         Method 
 & \multicolumn{2}{c}{Simple-CNN}& \multicolumn{2}{c}{Resnet18} & \multicolumn{2}{c}{Resnet50} \\
\hline $\#$Metric& PA & GA & PA & GA& PA & GA\\\midrule
         FedAvg & 93.22 & 86.99& 94.09 & 87.39& 94.28 & 88.21\\ \midrule
        Best Baseline &94.51 & 88.36& 95.18 & 88.85& 95.52 & 89.05\\ \midrule
         FedGELA~(ours) & 96.07& 90.03& 96.12 &90.61& 96.88 &91.23\\
        \bottomrule[2pt]
    \end{tabular}
    }}
    \label{tab:backbone}
\end{table}
For SVHN, CIFAR10, and CIFAR100, we conduct all experiments based on ResNet18~(modified by 32x32 input)~\citep{moon,fedskip,fedproto}. Here we adopt more backbones including Simple-CNN and ResNet50~\citep{moon,noniid2,fedskip} to verify the robustness of our method on different model structures. The Simple-CNN backbone has two 5x5 convolution layers followed by 2x2 max pooling (the first with 6 channels and the second with 16 channels) and two fully connected layers with ReLU activation (the first with 120 units and the second with 84 units. As shown in Table~\ref{tab:backbone}, we provide results of FedAvg, best baselines, and our FedGELA on  SVHN. As can be seen, with the model capacity increasing from Simple-CNN to ResNet50, the performance is slightly higher. Besides, no matter whether adopting any of the three backbones, our method FedGELA outperforms FedAvg and the best baselines.

\begin{table}[!t]
\caption{Performance of FedGELA compared with FedAvg and a range of state-of-the-art methods on two additional real-world challenges, namely FEMNIST and SHAKESPEARE. }
    \centering
\scalebox{1.0}{
    \setlength{\tabcolsep}{4.5mm}{
    \begin{tabular}{c|cccccccc}
        \toprule[2pt]
         Dataset 
 & \multicolumn{2}{c}{FedAvg}& \multicolumn{2}{c}{Best Baseline} & \multicolumn{2}{c}{FedGELA~(ours)} \\
\hline $\#$Metric& PA & GA & PA & GA& PA & GA\\\midrule
         FEMNIST & 67.02 & 59.54& 69.54& 61.22& 71.84 & 62.08\\ \midrule
        SHAKESPEARE &49.56 & 44.53& 51.66 & 47.29& 53.63 & 48.39\\ 
        \bottomrule[2pt]
    \end{tabular}
    }}
    \label{tab:morereal}
\end{table}
\subsection{Performance on more real-world datasets}\label{sec:morereal}
Except for Fed-ISIC2019 used in the main paper, we here additionally test FedGELA on two real-world federated datasets FEMNIST~\cite{femnist} and SHAKESPEARE~\cite{shakespeare} (two datasets also satisfy the PCDD setting) compared with all related approaches in the paper. FEMNIST includes complex 62-class handwriting images from 3500 clients and SHAKESPEARE is a next-word prediction task with 1129 clients. Most of the clients only have a subset of class samples. With help of LEAF~\cite{leaf}, we choose 50 clients of each dataset into federation and in each round we randomly select 10 clients into training. The total round is set to 20 and the model structure is a simple CNN for FEMNIST and a 2-layer LSTM for SHAKESPEARE, respectively. It can be seen in the Table~\ref{tab:morereal}, our method achieves best results of both personal and generic performance on the two real-world challenges.

\subsection{Limitations}
The design of our method is focused on constraining the classifier in the global server and in the local client by fixing the global classifier as a simplex ETF and locally adapting it to suit personal distribution, which means our method is proposed for federated classification tasks. But the spirit that treating each class or instance equally in global tasks while adapting to personal tasks the local can be applied to more than federated classification tasks. Fixing the classifier as a simple ETF might reduce the norm of stochastic gradients $G$ and benefit global convergence as introduced in Theorem~\ref{thm:global} and Theorem~\ref{thm:global_correct}. However, the limitation is that adapting the local classifier from ETF~($\W^L$) to adapted ETF~($\PHI_k\W^L$) will introduce additional cost $\|\PHI_k \W^L -\W^L\|$ as illustrated in the Theorem~\ref{thm:global},~\ref{thm:local},~\ref{thm:global_correct},~\ref{thm:local_correct}. To verify the influence of the cost and the effectiveness of our method on utilizing waste spaces and mitigating angle collapse of local classifier vectors, we conduct a range of experiments and record performance on both personal and generic and the corresponding angles.
\end{document}